\RequirePackage{fix-cm}
\documentclass[smallextended]{svjour3}       % onecolumn (second format)
\smartqed  % flush right qed marks, e.g. at end of proof
\usepackage{graphicx}
\usepackage{isabelle,isabellesym}
\usepackage{amssymb}
\usepackage{tikz-cd}

\usepackage{stmaryrd}
\usepackage{xcolor} 
\usepackage{float}
\usepackage{amsmath} % for logical symbols
\usepackage{bm} % for bold symbols
\tikzset{ampersand replacement=\&}
\isabellestyle{it}

\usepackage[backend = biber,
bibencoding=utf8,
style = authoryear-comp,
hyperref = true,
url = false,
doi = false,
giveninits=true,
maxbibnames = 8,
maxcitenames = 2,
uniquelist=false,
uniquename=init,
]{biblatex}

\def\axiomf{\cF_{\kern-.3em\typea\kern-.1em\typeb}}

\def\deq{{\;\colon\kern-.2em=\;}}

\def\lden{[\kern-0.18em[}    %oeffnende Denotatklammern
\def\rden{]\kern-0.18em]}  %schliessende Denotatklammern
\def\quotientspace#1#2{{#1/_{\kern-.5ex{#2}}}}

\def\termstructure#1#2{\cT\kern-.3ex\cE_{#1}(#2)}

\def\ambnormform#1{{#1}\hspace*{-1.1ex}\downarrow_{\kern-.2em\scriptscriptstyle *}} % ``ambiguous'' normal form

\def\typea{\alpha}
\def\typeb{\beta}

\def\NCalc{{\mathfrak N\kern -.1em\mathfrak K}}

\def\premdivider{{\mbox{$\quad$}}}
\def\@lineskipamount{4pt}
\def\lowerhalf#1{\hbox{\raise -0.8\baselineskip\hbox{#1}}}
\def\inruleanhelp#1#2#3{\setbox\tempa=\hbox{$\displaystyle{\mathstrut #2}$}%
                        \setbox\tempd=\hbox{$\; #3$}%
                        \setbox\tempb=\vbox{\vskip 2pt\halign{##\cr
        \mud{#1}\cr
        \noalign{\vskip\the\lineskip}%
        \noalign{\hrule height 0pt}%
        \rig{\vbox to 0pt{\vss\hbox to 0pt{\copy\tempd \hss}\vss}}\cr
        \noalign{\hrule}%
        \noalign{\vskip\the\lineskip}%
        \mud{\copy\tempa}\cr}}%
                      \tempc=\wd\tempb
                      \advance\tempc by \wd\tempa
                      \divide\tempc by 2 }
\def\inrulean#1#2#3{{\inruleanhelp{#1}{#2}{#3}%
                     \hbox to \wd\tempa{\hss \box\tempb \hss}}}
\def\inrulebn#1#2#3#4{\inrulean{#1\premdivider #2}{#3}{#4}}

\def\ian#1#2#3{{\lineskip\@lineskipamount\inrulean{#1}{#2}{#3}}} %premise, conc, name
\def\ibn#1#2#3#4{{\lineskip\@lineskipamount\inrulebn{#1}{#2}{#3}{#4}}}  %prem1, prem2, conc, name
\def\ianc#1#2#3{{\lineskip\@lineskipamount\lowerhalf{\inruleanhelp{#1}{#2}{#3}%
                   \box\tempb\hskip\wd\tempd}}}
\def\ibnc#1#2#3#4{{\lineskip\@lineskipamount\ianc{#1\premdivider #2}{#3}{#4}}}
\def\icnc#1#2#3#4#5{{\lineskip\@lineskipamount\ianc{#1\premdivider #2\premdivider #3}{#4}{#5}}}

\newbox\tempa
\newbox\tempb
\newdimen\tempc
\newbox\tempd
\def\mud#1{\hfil $\displaystyle{#1}$\hfil}
\def\rig#1{\hfil $\displaystyle{#1}$}

\def\phi{\varphi}

\def\cC{{\mathcal{C}}}

\def\cE{{\mathcal{E}}}
\def\cF{{\mathcal{F}}}

\def\cL{{\mathcal{L}}}

\def\cP{{\mathcal{P}}}
\def\cS{{\mathcal{S}}}
\def\cT{{\mathcal{T}}}
\def\cV{{\mathcal{V}}}

\def\my@ref#1#2#3{#1~\ref{#2:#3}}
\usepackage{mathpartir}

\def\impl{\rightarrow}
\def\typa{\shortrightarrow}
\def\land{\wedge}
\def\lor{\vee}
\def\turns{\vdash}
\def\pow{\cP}
\def\emptys{\emptyset}
\def\bool{o}

\newcommand{\Nat}{{\mathbb{N}}}

\def\makeheadbox{\relax}

\begin{document}

\title{Higher-order Logic as Lingua Franca}
\subtitle{Integrating Argumentative Discourse and Deep Logical Analysis}

%\titlerunning{Short form of title}        % if too long for running head

\author{David Fuenmayor%\orcidID{0000-0002-0042-4538}
	\and
	Christoph Benzm\"uller%\orcidID{0000-0002-3392-3093}
}

%\authorrunning{Short form of author list} % if too long for running head

\institute{David Fuenmayor \at
              Dep. of Mathematics and Computer Science, Freie
              Universit\"at Berlin, Berlin, Germany\\
              %Tel.: +123-45-678910\\
              %Fax: +123-45-678910\\
              \email{david.fuenmayor@fu-berlin.de}           %  \\
%             \emph{Present address:} of F. Author  %  if needed
           \and
           Christoph Benzm\"uller \at
            Dep. of Mathematics and Computer Science, Freie
              Universit\"at Berlin, Berlin, Germany\\
             \email{c.benzmueller@fu-berlin.de} 
}

\date{Received: date / Accepted: date}
% The correct dates will be entered by the editor

\maketitle

\begin{abstract}
  We present an approach towards the deep, pluralistic logical analysis of argumentative discourse that benefits from the application of state-of-the-art automated reasoning technology for classical higher-order logic. Thanks to its expressivity this logic can adopt the status of a uniform \textit{lingua franca} allowing the encoding of both formalized arguments (their deep logical structure) and dialectical interactions (their attack and support relations). We illustrate this by analyzing an excerpt from an argumentative debate on climate engineering.  
  Another, novel contribution concerns the definition of abstract, language-theoretical foundations for the characterization and assessment of shallow semantical embeddings (SSEs) of non-classical logics in classical higher-order logic, which constitute a pillar stone of our approach.
  The novel perspective we draw enables more concise and more elegant characterizations of semantical embeddings of logics and logic combinations, which is demonstrated with several examples.
 
% \PACS{PACS code1 \and PACS code2 \and more}
% \subclass{MSC code1 \and MSC code2 \and more}
\end{abstract}
\pagebreak

\section{Introduction}
\label{sec:introduction}	
Research presented at the 2\textsuperscript{nd} and 3\textsuperscript{rd} International Conferences on Logic and Argumentation (CLAR) applied higher-order automated and interactive theorem proving
to the deep logical analysis of rational arguments and argument networks \parencite{CH,C82}.

In the former paper we argued for an interpretive approach towards the deep, pluralistic logical analysis of argumentative discourse, termed \textit{computational hermeneutics}, amenable to partial mechanization using three kinds of automated reasoning technology: (i) theorem provers, which tell us whether a formalized claim logically follows from a set of assumptions; (ii) model finders, which give us (counter-)examples for formulas in the context of a background set of assumptions; and (iii) so-called ``hammers'', which automatically invoke (i) as to find minimal sets of relevant premises sufficient to derive a claim, whose consistency can later be verified by (ii). We exemplified this approach by employing implementations of (i-iii) for classical higher-order logic (HOL, \cite{SEP-TT}) within the \emph{Isabelle/HOL} proof assistant \parencite{Isabelle} to analyze several variants of Kurt G\"odel's ontological argument.\footnote{
	Ontological arguments (or proofs) are arguments for the existence of a Godlike being, common since centuries in philosophy and theology. More recently, they have attracted the attention of philosophers and logicians, not only because of their interesting history, but also because of their quite sophisticated logical structures.
}
Initially, the variants were reconstructed as networks of abstract nodes, which were mechanically tested for validity and consistency after adding or removing dialectical relations of attack and support. Later, each abstract node became `instantiated' by identifying it with a formula of a target logic, a higher-order modal logic in this case, and the experiments were repeated. 
Employing tools (i-iii), we showed that, e.g., consistency results for the abstracted arguments are not preserved at the instantiated level, i.e., after the internal logical structure of the argument nodes is provided. Drawing on this and other similar results, we argued that the analysis of non-trivial natural-language arguments at the abstract argumentation level is useful, but of limited explanatory power. Achieving proper explanatory power requires the extension of techniques from abstract argumentation with means for deep and logico-pluralistic semantical analysis using expressive logic formalisms, e.g.,~approaches inspired by Montague semantics \parencite{sep-montague-semantics} employing higher-order and modal logics, and, vice versa, deep logical methods for semantical analysis can become enriched at more abstract level by integrating them with contemporary argumentation frameworks \parencite{baroni2018handbook}.

In the second paper we further expanded these initial ideas and experiments towards the design of a systematic framework, thereby still largely omitting a proper theoretical exhibition, as contributed in the present article. 
In that paper we then focused from the beginning on instantiated argument networks and on the use of automated tools to support the process of reconstructing both individual argumentations and attack/support relations as \textit{deductive} arguments.  To that account we introduced another case study, this time formalizing and evaluating an excerpt from a quite contemporary and controversial discourse topic, \textit{ethical aspects of climate engineering} \parencite{CE}, in order to illustrate the generality of our ambitions in contrast to the previously studied, more philosophically-oriented arguments. 
Unsurprisingly, as in most debates, arguments can generally be seen as enthymemes in need for missing, implicit premises; apparently this seems even more so the case for `real-life' debates than for philosophical arguments. We illustrated how the utilization of automated reasoning technology for expressive higher-order logics has realistic prospects in the analysis of argumentative discourse `in the wild'. In particular, our results suggested that this technology can be very useful to help in the reconstruction of argument networks using structured, deductive approaches,\footnote{
	Our reason for choosing a deductive approach originally had a technical motivation: the base logic provided (off-the-shelf) in \textit{Isabelle/HOL} is classical (structural, monotonic, etc.), and, as we will see, SSEs in the default setting reuse the classical consequence relation of the meta-logic HOL. However, SSEs of non-monotonic logics are also possible, e.g., by explicitly modeling a non-monotonic conditional operator \parencite{J31}
and by defining a corresponding (non-monotonic) consequence relation which satisfies the deduction meta-theorem. For the time being we won't be pursuing such an approach, aiming at simplicity of exposition, and since this often takes a toll on the performance of automated tools. In this respect we have chosen to treat arguments as deductions, thus locating all fallibility of an argument in its (often implicit) premises.} 
such as \textit{ABA} \parencite{dung2009assumption} and \textit{deductive argumentation} \parencite{BH},  and also to help identify implicit and idle premises in arguments, as addressed in our previous work.

The case studies conducted in both papers were formalized and encoded, using combinations of quantified modal logics, in the \emph{Isabelle/HOL} proof assistant, which features a classical higher-order logic (HOL) as its logical foundation. HOL is a conservative extension of Church's \textit{simple type theory} (STT), a logic of functions formulated on top of the simply typed lambda-calculus, which also provides a foundation for functional programming.
In order to turn \textit{Isabelle/HOL} into a flexible modal logic reasoner, we have adopted the \textit{shallow semantical embeddings} (SSE) approach \parencite{J41}, which harnesses the expressive power of higher-order logic as a \textit{meta-language} allowing the faithful encoding of Kripke-like semantics for quantified modal logics (among other non-classical logics) in {STT/HOL}, thereby turning higher-order theorem proving systems into universal reasoning engines.

Putting a focus on the `real-life' case study from the second paper, this article merges the ideas and developments from our prior conference papers and presents them in a coherent, self-contained manner. 
In addition, we further extend our prior work  by adding a self-contained exposition of the main theoretical underpinnings of the SSE approach using the conceptual framework of formal signatures, languages and their morphisms -- a perspective inspired by, and linking to, algebraic methods in logic \parencite{handbook-logic-cs5,carnielli2008analysis} --,  while at the same time striving for a most natural, self-contained exposition, e.g.~not relying on the reader's previous knowledge of categorial or algebraic notions. The original case study on the ethical aspects of climate engineering has been slightly adapted for the sake of illustrating the theoretical notions as newly introduced.  

Paper structure: In Section~\ref{sec:linguafranca} we argue for the usefulness of a higher-order logic as a sort of \textit{lingua franca} allowing us to `glue together' arguments encoded in different logics. In the next two sections we prepare the ground for the definition of a \textit{shallow semantical embedding} (SSE) of a target logic into a host (meta-)logic. In Section~\ref{sec:notions} we introduce some basic conceptual machinery (illustrated with numerous examples) allowing us to characterize both host and target languages in terms of their signatures, while abstracting away from their grammars. Section~\ref{sec:sse-char} offers a semantic justification for this move by introducing the notion of a \textit{derived signature}, which enables us to characterize SSEs as: (i) specifically constrained fragments of a higher-order host language, and (ii) (conservative) translations from the target language into this fragment of the host language. This section then illustrates with a wide array of examples how to employ the SSE technique to characterize different kinds of target logics this way.
Section~\ref{sec:arg} introduces some useful notions from formal argumentation to be employed later in our case study and shows how  to encode them in the higher-order logic of \textit{Isabelle/HOL} by using this technique.
Section~\ref{sec:case-study} presents our illustrative `real-life' case study on the subject of \textit{climate engineering}.\footnote{
	Sources for this case study have been made available online (\url{https://github.com/davfuenmayor/CE-Debate}). We encourage the interested reader to try out (and improve on) this work.
}
Section \ref{sec:conclusion} concludes the paper and discusses ongoing and future work.

%%% Local Variables: 
%%% TeX-master: "root"
%%% End: 
\section{Classical Higher-Order Logic as a Lingua Franca}
\label{sec:linguafranca}

The need for combining heterogeneous, expressive logical formalisms for the analysis of argumentative discourse is manifest in view of the richness of natural language phenomena.
As we see it, the problem is less the lack of logical systems to represent those diverse perspectives, but rather the issue of bringing them coherently under the same roof. In other words, the actual lack is a \textit{lingua franca} by means of which we can (i) flexibly combine logics, as required for the appropriate formalization of non-trivial normative arguments, and (ii) enable the articulation of inter-logical dialectical relations; e.g.,~\textit{how can arguments formalized using different logics actually attack or support each other?}.

The proposed solution relies on the adoption of classical higher-order logic (HOL) \textit{as a metalanguage} into which the logical connectives of (a combination of) target logics can be `translated' or `embedded'. This approach, termed \textit{shallow semantical embeddings} (SSE)~\parencite{J41,J23}, has quite interesting practical applications, and it supports the reuse of existing reasoning infrastructure for first-order (FO) logic and higher-order (HO) logic for seamlessly combining and reasoning with different quantified classical and non-classical logics---including modal, deontic, and paraconsistent logics as illustrated below---many of which are well suited for normative reasoning applications.\footnote{
	The SSE technique, which has  become the pillar stone of the LogiKEy~\parencite{J48} framework and methodology for designing normative theories in ethical and legal reasoning, has already demonstrated its relevance for research, education and application; cf.~the examples at \url{logikey.org}. It has e.g.~been exploited successfully in prior work on the logical analysis of argumentative discourse~\parencite{J38,B19,C77,C82} and also in computational metaphysics \parencite{J47}.
}
 Moreover, it supports a \textit{logico-pluralistic} approach towards the formalization of arguments, indeed blurring the line between logical and extralogical, respectively~syncategorematic and categorematic, expressions.

\begin{figure}[t]
	\centering{
		\includegraphics[width=0.6\textwidth]{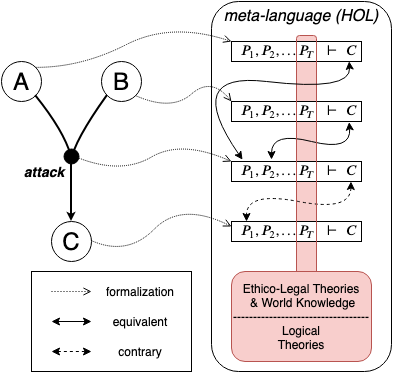}
		\caption{Analyzing a joint attack relation.}
		\label{figArgAnalysis}}
\end{figure}

An illustrative situation is depicted in Fig.~\ref{figArgAnalysis}, where two arguments, formalized in, say, logics $\cL_1$ and $\cL_2$, jointly attack a third argument, formalized in logic $\cL_3$. 
The joint attack relation is itself modeled and encoded as an argument employing the logic combination $\cL_1+\cL_2+\cL_3$. This approach to the computer-supported logical analysis of argumentative discourse is further explained in Section~\ref{sec:case-study} using 
the already mentioned `real-life' debate on the ethical aspects of climate engineering.
Further notice that, since our approach presupposes the logical validity and non-circularity of the formalized arguments, unless we have solid evidence of the contrary,\footnote{
	In particular, following the interpretive ``principle of charity'' \parencite{Davidson}, we aim at formalizations which render the argument as logically valid, while having a consistent and minimal set of assumptions. These actions are to be repeated until arriving at a state of \emph{reflective equilibrium}: a state where our arguments and claims have the highest degree of coherence and acceptability according to syntactic and, particularly, inferential criteria of adequacy; cf.~\textcite{CH,B19}.	
}
there will be, in most non-trivial cases, a need for additional, tacit premises; cf.~$P_T$ in Fig.~\ref{figArgAnalysis}. These additional premises can be of a different nature. Some may be part of underlying \textit{logical theories} and thus correspond to axioms (such as $K$, $D$, $T$, $4$, etc.)  or definitions from the SSE for the target logic (cf.~\cite[\S7.1]{J48}), while others may correspond to formalized principles from e.g.~domain theories or general world knowledge (cf.~\cite[\S7.2]{J48}). They may also correspond to unstated, `implicit' assumptions specific for the argument at hand.

In the next section we will explore some notions allowing us to articulate theoretically our notion of SSE. We start by illustrating propositional signatures, before generalizing towards the corresponding higher-order notions.

%%% Local Variables: 
%%% TeX-master: "root"
%%% End: 
\section{Preliminaries: Signatures, Grammars and Languages}
\label{sec:notions}

It is customary to define propositional languages over a given set of propositional symbols (sometimes called a ``signature''), while the logical connectives are introduced by the language-generating mechanism or grammar. As an illustration, a typical textbook approach for defining a modal language starts by introducing a ``signature'' set $\cP=\{p^n\}_{0\leq n< m}$
consisting of $m$ propositional atoms. Subsequently, the modal language becomes defined inductively by a  grammar such as the following, where $p$ ranges over elements of $\cP$: 
\[ \phi :=~p~|~\bot~|~\neg\phi~|~\phi_1\lor\phi_2~|~\Box\phi\]

By contrast, presenting a logico-pluralistic approach compels us to do things slightly differently. Our main tactic in this section consists in shifting logical connectives from the grammar into the signature of a language, which  is a principle particularly well known in the literature on HO logics. This move facilitates definitions and assessments at a higher level of abstraction. Also note that, in the present account, substitutions do not operate on propositional symbols or constants, and we will explicitly introduce variables for that purpose.

Before getting started some remarks are pertinent: The present section aims at building a conceptual framework and is thus rich in definitions. For each of them we have provided one or several examples regarding well-known modal-like systems which have been previously encoded using the SSE approach. It is important to note that SSEs of non-modal systems are also possible \parencite{J46}, their properties have just been less studied. Anyhow, the `real-life' illustration for the introduced notions occurs in our case study in Section~\ref{sec:case-study}. In this exposition we strive for a middle path between readability and rigor. We may switch between prefix and infix notation without warning, and we may omit parenthesis when they can be easily inferred in context. 

\subsection{Propositional Signatures}
\label{sec:notions:prop-sig}

We start with discussing propositional languages since, in spite of their simplicity, they readily provide a perfect conceptual bridge towards the sort of HO languages (based on functional STT) we utilize in the SSE approach. We will gracefully skip FO logic, since, for practical purposes, formulating FO languages as a fragment of an adequate HO language suffices.

\begin{definition}[P-Signature] 
A propositional signature, termed \textit{P-signature}, is a tuple $\cS=\langle \cC, \cP_0 \rangle$, where $\cC$ is a non-empty, denumerable set of disjoint sets $\{\cC_k\}_{k\in \Nat}$, and $\cP_0$ is a possibly empty, denumerable set $\{p^n\}_{n\in \Nat}$. The elements of each $\cC_k$ are symbols called \emph{k-ary} connectives 
and are always given a fixed (intended) interpretation. The elements of $\cP_0$ are symbols, called \emph{propositional constants}, whose denotation varies in each interpretation. P-signatures can be partially ordered. We have $\langle \cC^1, \cP^1_0 \rangle\leq \langle \cC^2, \cP^2_0 \rangle$ iff $\cC^1_k \subseteq \cC^2_k$ and $\cP^1_0 \subseteq \cP^2_0$ (analogously for $<$ and $\subset$). We define the union (intersection) of P-signatures as the union (intersection) of their respective components: $\cS^1 \cup(\cap)\,\cS^2 = \langle \{\cC_k\}_{k\in \Nat},\cP_0 \rangle$, with $\cC_k=\cC^1_k \cup(\cap)\,\cC^2_k$ and $\cP_0 = \cP^1_0 \cup(\cap)\,\cP^2_0$.
\end{definition}

$\cC$ and $\cP_0$ can be seen as \emph{logical}, resp.~\emph{extralogical},  base expressions of a propositional language, and elements of $\cP_0$ can be seen as having arity zero. In the sequel, to avoid cluttering in our notation, all sets $\cC_k$ which are not explicitly mentioned are assumed to be empty. Signatures can be minimal or non-minimal.

\begin{example}[CPL]\label{example-sig-cl}
An exemplary P-signature for classical propositional logic (CPL) with $m$ propositional constants is: $\cS_{\text{CPL}}=\langle \{\cC_k\}_{k\leq 2},\{p^n\}_{n<m} \rangle$, with $\cC_0=\emptys$, $\cC_1=\{\neg\}$ and $\cC_2=\{\land\}$. Notice that
$\cS_{\text{CPL}}$ can be considered minimal; other connectives can be defined the usual way.
\end{example}

\begin{example}[IPL]\label{example-sig-il}
An exemplary P-signature for intuitionistic propositional logic (IPL) with $m$ propositional constants is: $\cS_{\text{IPL}}=\langle \{\cC_k\}_{k\leq 2},\{p^n\}_{n<m} \rangle$, with $\cC_0=\{\bot\}$, $\cC_1=\{\neg\}$ and $\cC_2=\{\land,\lor,\impl\}$. This signature is non-minimal since, e.g.,~$\neg$ and $\lor$ could be defined based on the others. Moreover, note that $\cS_{\text{CPL}} < \cS_{\text{IPL}}$.
\end{example}

\begin{example}[ML]\label{example-sig-modal}
An exemplary P-signature for multi-modal propositional logic with denumerable propositional constants is: $\cS_{\text{ML}}=\langle \{\cC_k\}_{k\leq 2},\{p^n\}_{n\in\Nat} \rangle$, with $\cC_0=\emptys$, $\cC_1=\{\neg\}\cup \{\Box^n\}_{n\in\Nat}$ and $\cC_2=\{\land\}$. Note that $\cS_{\text{CPL}}=\cS_{\text{ML}}\cap\cS_{\text{IPL}}$.
\end{example}

\subsection{Propositional Languages}
\label{sec:notions:prop-lang}

Utilizing the notion of P-signatures, we introduce a technique to define different propositional languages inductively at a higher level of abstraction. A relevant objective thereby is to keep the grammar fixed, so that only changes in the signature can account for differences in the defined languages. 

In addition to the symbols defined in the signature, we from now on assume a \textit{fixed}, denumerable set $\cV_0=\{v^n\}_{n\in\Nat}$ of schema variables. 

\begin{definition}[P-Language]  
The propositional language $L=\cL^P(\cS)$ over the P-signature $\cS=\langle \{\cC_k\}_{k\in\Nat}, \cP_0 \rangle$ is the smallest set such that:
\begin{enumerate}
\normalfont
	\item $v,p,c\in \cL^P(\cS)$ for every $v\in \cV_0$, $p\in\cP_0$, and $c\in\cC_0$;
	\item $c(\phi_1,\dots,\phi_k) \in \cL^P(\cS)$, whenever $c\in \cC_k$ ($k\geq 1$) and $\phi_1,\dots,\phi_k \in \cL^P(\cS)$.
\end{enumerate}

\end{definition}

\begin{example}
Each P-signature $\cS$ presented in Ex.~\ref{example-sig-cl}--\ref{example-sig-modal} induces a corresponding propositional language $\cL^P(\cS)$ in the manner exposed above.
\end{example}

\begin{definition}[Language Fragment]
	\label{def-fragment}
Let $L^1$ and $L^2$ be two languages, i.e.,~there exist signatures $\cS^1$ and $\cS^2$ such that $L^1=\cL^P(\cS^1)$ and $L^2=\cL^P(\cS^2)$. We say that $L^1$ is a \textit{fragment} of $L^2$, and write $L^1 \leq L^2$, iff $L^1 \subseteq L^2$. In a similar vein, $L^1$ is a \emph{proper fragment} of $L^2$, noted $L^1 < L^2$, iff $L^1 \subset L^2$. 
\end{definition}	 

\begin{example}
Language $\cL^P(\cS_{\text{CPL}})$ is a fragment of language $\cL^P(\cS_{\text{IPL}})$.
\end{example}

Notice that, in this exposition, languages are not simply sets of well-formed formulas. Languages must be systematically generated (e.g.,~induced by a signature), and  they are, among others, not closed under intersection or formula-removal. This is why we won't define any set-theoretical operations on languages. Moreover, note that the present characterization of a language fragment is actually applicable to all languages induced by signatures, and not only propositional ones.

\begin{remark}[On derived connectives]
Presentations of propositional languages often include so-called \emph{derived connectives}, which are actually abbreviations (``syntactic sugar'') for sequences of connectives; e.g., in modal logics the operator $\Diamond$ may simply abbreviate $\neg\Box\neg$. Note, however, that $\neg\Box\neg$ is not a well-formed expression, and thus there is no proper way to characterize derived connectives using P-signatures and the language-generating grammar presented above. This contrasts with higher-order languages, introduced in Section~\ref{sec:notions:ho-lang}, which are equipped with a suitable functional abstraction mechanism. 	
\end{remark}

%================= HIGHER-ORDER SIGNATURES AND LANGUAGES ================
\subsection{Higher-order Signatures}
\label{sec:notions:ho-sig}

In this section we introduce (functional) HO signatures as a straightforward generalization of propositional signatures.

\begin{definition}[Functional Type]
	\label{def-functional-type}
	We inductively define the following denumerable set $\tau$:
	\begin{itemize}
		\normalfont
		\item $\iota_n \in \tau$ for every $n\in \Nat$;
		\item $\alpha\typa\beta \in \tau$, whenever $\alpha,\beta \in \tau$.
	\end{itemize}
The elements of set $\tau$ are called \textit{functional types}. They will play analogous roles to the arities in P-signatures. We will use the following aliases for commonly used (base) types: $\bool$ for $\iota_0$, $w$ for $\iota_1$, and $e$ for $\iota_2$. Note that $\typa$ associates to the right, so that~$\alpha\typa\beta\typa\gamma$ is shorthand for $\alpha\typa(\beta\typa\gamma)$. Moreover, $\alpha^n\typa\beta$ is shorthand for  $\alpha\typa\alpha\typa\ldots_{~(n-times)}\ldots\typa\beta$.
\end{definition}

We do not introduce product types, since they can (and will) be emulated using functional types. For example,~the product type $(\alpha\times\beta)\typa\eta$ corresponds to the functional type  $\alpha\typa(\beta\typa\eta)$.\footnote{As an illustration, in a functional language, a binary operator such as $+ :\Nat\times\Nat\typa\Nat$ acquires the type: $\Nat\typa(\Nat\typa\Nat)$. This way we have, e.g., $((+~3)~4)=7$ as a wff. The underlying notion is known as Sch\"onfinkelization or Currying in the literature.}

\begin{definition}[F-Signature]
A functional type-theoretical signature, termed \textit{F-signature},  is a tuple $\cS=\langle \cC, \cP \rangle$, where $\cC$ is a non-empty denumerable set of (disjoint) sets $\{\cC_\alpha\}_{\alpha \in \tau}$, and where $\cP$ is a (possibly empty) denumerable set of (disjoint) sets $\{\cP_\alpha\}_{\alpha \in \tau}$. The elements of each $\cC_\alpha$ are symbols called \textit{connectives of type $\alpha$}. They are always given a fixed (intended) interpretation. The elements of each $\cP_\alpha$ are also symbols and they are called \textit{parameters (or functional constants) of type $\alpha$}. Their denotation varies in each interpretation.
F-signatures can be \textit{partially ordered}: we have $\langle \cC^1, \cP^1 \rangle\leq \langle \cC^2, \cP^2 \rangle$ iff $\cC^1_\alpha \subseteq \cC^2_\alpha$ and $\cP^1_\alpha \subseteq \cP^2_\alpha$ (analogously for $<$ and $\subset$).
Moreover, we define: $\cS^1 \cup(\cap)\,\cS^2 = \langle \{\cC_\alpha\}_{\alpha \in \tau},\{\cP_\alpha\}_{\alpha \in \tau}\rangle$, with $\cC_\alpha=\cC^1_\alpha \cup(\cap)\,\cC^2_\alpha$ and $\cP_\alpha = \cP^1_\alpha \cup(\cap)\,\cP^2_\alpha$.
\end{definition}

Observe that F-signatures do not feature $\lambda$-expressions, which appear instead in the language-generating grammar as introduced below in Def.~\ref{def-f-lang}. In doing this we aim at establishing a common ground for the analysis of signatures and languages at a higher level of abstraction. 

\begin{example}[CPL-F]\label{example-sig-fttcl}
An exemplary F-signature for classical propositional logic (CPL) with $m$ propositional constants is: $\cS_{\text{CPL-F}}=\langle \{\cC_{o\typa o},\cC_{o\typa o\typa o} \},\{\cP_o\} \rangle$, with $\cC_{o\typa o}=\{\neg\}$, $\cC_{o\typa o\typa o}=\{\land\}$ and $\cP_o=\{p^n\}_{n<m}$. Compare this F-signature with the corresponding P-signature as introduced in Ex.~\ref{example-sig-cl}. 
\end{example}

In Ex.~\ref{example-sig-fttcl} the types have been carefully chosen to reflect the intended interpretation of the symbols in the signature. Type $o$ is intended to represent truth-values, and types $o\typa o$ and $o\typa o\typa o$ are intended to correspond to unary and binary logical connectives, respectively. For the time being our exposition focuses on syntactic aspects. Semantical issues may gradually crop up though, as we (informally) point out intended interpretations, and they will become conspicuous when we discuss \textit{derived signatures} in Section~\ref{sec:sse-char:derived}.

\begin{example}[FOL-F]\label{example-sig-fttfol}
	An exemplary F-signature for FO logic  with equality and
%without relation symbols, but 
countably many function symbols, $m$ individual constant symbols, and $n$ unary predicate symbols is:
	$\cS_{\text{FOL-F}}=\langle \{\cC_{o\typa o},\cC_{o\typa o\typa o},\allowbreak\cC_{e\typa e},\allowbreak\cC_{e\typa e\typa o},\allowbreak\cC_{(e\typa o)\typa o} \},\allowbreak\{\cP_e,\cP_{e\typa o}\} \rangle$, with $\cC_{o\typa o}=\{\neg\}$,
	$\cC_{o\typa o\typa o}=\{\land,\lor,\impl\}$, $\cC_{e\typa e}=\{f^k\}_{k\in \Nat}$, $\cC_{e\typa e\typa o}=\{=^e\}$, $\cC_{(e\typa o)\typa o}=\{\Pi^e\}$, $\cP_e=\{p_e^k\}_{k<m}$, and $\cP_{e\typa o}=\{p_{e\typa o}^k\}_{k<n}$.\footnote{\label{footnote-label-types}Note that, when we have parameters for each type, we may use their type as an additional label. This is to avoid defining a numbering mechanism for types.}
\end{example}

In Ex.~\ref{example-sig-fttfol} the type $e$ is intended to denote individuals, so that $e\typa e$ becomes the type for (unary) functions over individuals.
In particular, the symbol $\Pi^e$, of type $(e\typa o)\typa o$, has a fixed intended interpretation as a special second-order (SO) predicate assigning true to those unary predicates which are true of every individual (of type $e$).
	More generally, $\Pi^{\alpha}$, for $\alpha\in\tau$ 
of type $(\alpha\typa o)\typa o$, has a fixed interpretation as a special predicate assigning true to those 
predicates of type $\alpha\typa o$ which are true of all of their arguments (of type $\alpha$).
We use $\forall x_\alpha.\phi$ as shorthand 
for $\Pi^\alpha(\lambda x_\alpha.\phi)$ and, analogously, $\exists x_\alpha.\phi$ as shorthand for $\neg\Pi^\alpha(\lambda x_\alpha.\neg\phi)$.

\begin{example}[SOL]\label{example-sig-fttsol} \sloppy
	An exemplary F-signature for SO logic with relations and with two SO predicate constants $H^1$ and $H^2$, but without any functions symbols is:
	$\cS_{\text{SOL}}=\langle \{\cC_{o\typa o},\cC_{o\typa o\typa o}\} \cup \{\cC_{(e^n\typa o)\typa o}\}_{n\in \Nat} \cup \{\cC_{((e\typa o)^n\typa o)\typa o}\}_{n\in \Nat},\allowbreak\{\cP_{(e\typa o)\typa o}\} \rangle$, with $\cC_{o\typa o}=\{\neg\}$, $\cC_{o\typa o\typa o}=\{\land,\lor,\impl\}$, $\cC_{(e\typa o)\typa o}=\{\Pi^{e}\}$, $\cC_{((e\typa o)\typa o)\typa o}=\{\Pi^{(e\typa o)}\}$,
	and $\cP_{(e\typa o)\typa o}=\{H^1, H^2\}$.
\end{example}

\begin{example}[STT]\label{example-sig-fttstt1}
	An exemplary F-signature for Church's simple type theory (STT) with countably many parameters (for all types $\alpha\in\tau$) is:\linebreak
	$\cS_{\text{STT}}=\langle \{\cC_o,\cC_{o\typa o},\cC_{o\typa o\typa o}\} \cup \{\cC_{(\alpha^n\typa o)\typa o}\}_{\alpha\in\tau,n\in\Nat}\,,\{\cP_\alpha\}_{\alpha\in\tau} \rangle$, with $\cC_o=\{T,F\}$, $\cC_{o\typa o}=\{\neg\}$, $\cC_{o\typa o\typa o}=\{\land,\lor,\impl\}$, $\cC_{(\alpha\typa o)\typa o}=\{\Pi^{\alpha}\}$,
	and $\cP_\alpha=\{p^k_\alpha\}_{k\in\Nat}$.
\end{example}

\begin{example}[Minimal STT]\label{example-sig-fttstt2}
	A minimal F-signature for STT, cf.~\textcite{andrews02,SEP-TT}, with countably many parameters is: 
	$\cS_{\text{STT}^=}=\langle \{\cC_{\alpha\typa\alpha\typa o}\}_{\alpha\in\tau}\,,\{\cP_\alpha\}_{\alpha\in\tau} \rangle$, with $\cC_{\alpha\typa\alpha\typa o}=\{=^\alpha\}$, and $\cP_\alpha=\{p^k_\alpha\}_{k\in\Nat}$.
\end{example}

Our above examples are motivated by the idea to provide uniform, comparable and combinable characterizations of different logic languages, and this will be very useful for supporting a more elegant presentation of the SSE technique. We provide some further useful examples in this direction.

\begin{example}[Relational structures]
	\label{example-sig-relstr}
	Consider the following F-signature for the FO language of relational structures. Such a language can be used, in particular, to express and encode the semantics of propositional normal modal logics using Kripke semantics (type $w$ is then associated with possible worlds).
	$\cS_{\text{FORS}}=\langle \{\cC_o,\cC_{o\typa o},\allowbreak\cC_{o\typa o\typa o},\allowbreak\cC_{(w\typa o)\typa o} \},\allowbreak\{\cP_{w\typa o}, \cP_{w\typa w\typa o}\} \rangle$, with $\cC_o=\{T,F\}$, $\cC_{o\typa o}=\{\neg\}$,
	$\cC_{o\typa o\typa o}=\{\land,\lor,\impl\}$, $\cC_{(w\typa o)\typa o}=\{\Pi^w\}$, $\cP_{w\typa o}=\{p^k\}_{k\in\Nat}$, and $\cP_{w\typa w\typa o}=\{R^k\}_{k\in\Nat}$. 
\end{example}

\begin{example}[Neighborhood structures]
	\label{example-sig-nbhdstr} \sloppy
	Consider the following F-signature for a SO language of neighborhood structures. Such a language can be used to express and encode different sorts of semantics for propositional non-normal modal logics. This example features two sorts of (``neighborhood'') functions for monadic and dyadic operators. ($wo$ is used  as shorthand for type ${w\typa o}$ to improve readability.)
	$\cS_{\text{SONS}}=\langle \{\cC_{o\typa o},\cC_{o\typa o\typa o},\cC_{wo\typa o}, \cC_{(wo\typa o)\typa o}\},\{\cP_{wo}, \cP_{wo\typa wo}, \cP_{wo\typa wo\typa wo}\} \rangle$, with $\cC_{o\typa o}=\{\neg\}$,
	$\cC_{o\typa o\typa o}=\{\land,\lor,\impl\}$, $\cC_{wo\typa o}=\{\Pi^w\}$, $\cC_{(wo\typa o)\typa o}=\{\Pi^{wo}\}$, $\cP_{wo}=\{p^k\}_{k\in\Nat}$, $\cP_{wo\typa wo}=\{N_1^k\}_{k\in\Nat}$, and $\cP_{wo\typa wo\typa wo}=\{N_2^k\}_{k\in\Nat}$.
\end{example}

\begin{remark}
	\label{remark-neighborhood}
	We give some further explanations. In semantical approaches to (non-normal) modal and deontic logics (cf.~\textit{minimal} semantics \parencite{chellas1980modal} or \textit{neighborhood} semantics \parencite{pacuit2017neighborhood}) so-called \textit{neighborhood functions} are usually introduced, in set-theoretical terms, as functions $N(w): W \typa \pow(\pow(W))$ (where $W$ is the domain set and $\pow(W)$ its powerset) that assign to each world $w$ a set of sets: its neighborhood. This would correspond to the functional type ${w\typa (w\typa o)\typa o}$. It is evident that $N$ can be associated with a function $N^*(\phi): \pow(W)\typa \pow(W)$, where $N^*(\phi)$ stands for the set of worlds to which $\phi$ gets assigned by $N$. This corresponds to the functional type ${(w\typa o)\typa (w\typa o)}$. The same rationale can be applied to dyadic neighborhood functions, which thus get the type ${(w\typa o)\typa (w\typa o)\typa (w\typa o)}$. Note that $\Pi^w$ and $\Pi^{(w\typa o)}$ (for quantifiers) range over worlds and propositions (sets of worlds) respectively.
\end{remark}

\subsection{Higher-order Languages}
\label{sec:notions:ho-lang}

Analogous to the propositional case, we define different HO languages inductively at a higher level of abstraction. Again, the grammar is fixed for all languages and only the signature changes from case to case. We thereby assume a \textit{fixed}, denumerable set $\cV=\{v^k_\alpha\}_{\alpha\in\tau,k\in\Nat}$ of variables.

\begin{definition}[F-Language]
	\label{def-f-lang}
The (HO) \textit{functional language $L=\cL^F(\cS)$} over F-signature $\cS=\langle \cC,\cP \rangle$ is the smallest set defined inductively as:
\begin{enumerate}
	\normalfont
	\item $v,p,A\in \cL^F(\cS)$ for every $v\in \cV$, $p\in \cP_{\alpha\in\tau}$, and $A\in \cC_{\alpha\in\tau}$;
	\item $(A_{\alpha\typa\beta}~B_\alpha)_\beta \in \cL^F(\cS)$ whenever $A,B \in \cL^F(\cS)$ for all $\alpha, \beta$ in $\tau$;
	\item $(\lambda x_\alpha.A_\beta)_{\alpha\typa\beta}\in \cL^F(\cS)$ whenever $A \in \cL^F(\cS)$ and $x\in \cV$ for all $\alpha, \beta$ in $\tau$.
\end{enumerate}	
The elements of an F-language, i.e., its well-formed formulas, are called \emph{terms}; terms of type $o$ are sometimes called formulas.

We define the \textit{grounded}\footnote{For want of a better word, we call terms ``grounded'' when they are `parameter-free'. Note that this does not coincide with the standard notion of ``grounded'' in the literature.
}
language $\cL^F_{Gr}(\cS)$ over F-signature $\cS=\langle \cC,\cP \rangle$ as the language over $\cS^{Gr}=\langle \cC,\emptys \rangle$.
A term is called \emph{closed} if it does not contain any free variables.\footnote{Free variables are those which appear either outside any $\lambda$-expression, or unbound inside some $\lambda$-expression; cf.~\textcite{SEP-TT} for a detailed exposition.}
We define the \emph{language of closed terms} of a language $L=\cL^F(\cS)$ as the subset of $L$ consisting of all of its closed terms; this language is noted $Closed(L)$. Moreover, we lift the definition of \textit{(proper) fragment} relation to F-languages by adapting Def. \ref{def-fragment} ($L^1 \leq L^2$ iff $L^1\subseteq L^2$ and $L^1 < L^2$ iff $L^1\subset L^2$).
\end{definition}

\begin{example}
	\label{example-stt-cpl}
	The F-signature presented in Ex.~\ref{example-sig-fttcl} induces a corresponding language $\cL^F(\cS_{\text{CPL-F}})$. This language can have the same expressivity as the language $\cL^P(\cS_{\text{CPL}})$  induced by its counterpart in Ex.~\ref{example-sig-cl} if we omit the last item (concerning $\lambda$-abstraction) in the Def.~\ref{def-f-lang} above.
\end{example}

\begin{example}
	\label{example-stt-fol} \sloppy
	The F-signature presented in Ex.~\ref{example-sig-fttfol} induces a language $\cL^F(\cS_{\text{FOL-F}})$,  which corresponds to an extension of FO logic called \emph{extended first-order} logic  \parencite{brown2009extended}, which restricts quantification and equality to base type $e$ but retains $\lambda$-abstractions and HO variables. Notice that we can restrict the first and the last item in the Def.~\ref{def-f-lang} above as suited to our purposes.
\end{example}

\begin{example}
	\label{example-stt-lang} \sloppy
	The F-signatures presented in Ex.~\ref{example-sig-fttstt1} and Ex.~\ref{example-sig-fttstt2} induce two languages $\cL^F(\cS_{\text{STT}})$ and $\cL^F(\cS_{\text{STT$^=$}})$ resp. The latter language is indeed at least as expressive as the first one, since the family of binary connectives $=^\alpha$ can be used to define all others. We refer the reader to \textcite{andrews02} for a discussion.
\end{example}

\begin{example}
	\label{example-relstr-lang} \sloppy
	The F-signature presented in Ex.~\ref{example-sig-relstr} induces a FO language $\cL^F(\cS_{\text{FORS}})$ which can be used to articulate semantic conditions for classical normal modal logics in the style of Kripke semantics. Note that special measures can be taken to remain inside a strict FO fragment, if desired.\footnote{See Ex.~\ref{example-stt-cpl} and Ex.~\ref{example-stt-fol}. We see, however, an advantage in this increased expressivity; as it allows us to go beyond well-known limitations of Kripke semantics for modal logics, e.g., by encoding second-order frame conditions. As an example, current work on the SSE for provability logic bears witness to this claim, since G\"odel-L\"ob's axiom $GL: \Box(\Box\phi\impl\phi)\impl\Box\phi$ imposes second-order restrictions on the frame's accessibility relation, such as finiteness or converse well-foundedness \parencite{boolos1995logic}.}
\end{example}

\begin{example}
	\label{example-nbhdstr-lang} \sloppy
	The F-signature presented in Ex.~\ref{example-sig-nbhdstr} induces a SO language $\cL^F(\cS_{\text{SONS}})$ which can be used to articulate semantic conditions for classical non-normal modal logics using a neighborhood semantics \parencite{chellas1980modal,pacuit2017neighborhood}.
\end{example}

%=================  FRAGMENTS ====================
\section{SSEs as Language Fragments and Translations}
\label{sec:sse-char}
The notions introduced in this section are stated in a language- and signature-agnostic way (e.g., dropping superscripts). As will become clear, they make most sense for HO signatures and languages, as discussed previously. It is up to the concerned reader to rephrase them for the propositional (or FO) setting if desired, and to the extent that their reduced expressivity may allow.

\subsection{Derived Signatures}
\label{sec:sse-char:derived}
A signature in our framework is composed of (type-indexed) sets of symbols, further divided into connectives $\cC$ and parameters $\cP$. Notice that these symbols adopt a double role as both atomic building blocks and terms (i.e., well-formed formulas) of a language.
An interesting question thus is whether terms, in general, can act as atomic building blocks in the construction of languages; or more specifically, whether terms, being symbols\footnote{In the current context, we understand symbols as sequences of characters intended to stand for (represent) something else (their denotation).}
too, can also adopt the role of connectives in signatures.
As we will see, we can give this question a partially positive answer. For this we introduce the notion of derived signatures.

Informally, derived signatures are signatures where \textit{closed} terms can play the role of logical connectives. They are called \emph{derived} because they rely on an already existing language, itself induced by a different, `original' signature. In such cases we say that signature $\cS^D$ has been \emph{derived} from signature $\cS$ and introduce the relation $derived(\cS^D,\cS)$ to indicate this. We thus say that $\cS^D$ is a \emph{derived} signature. Signatures which are not derived will be called \emph{primitive} (like the ones discussed so far). For the sake of further analysis, the current notion will be divided into two categories: \emph{rigidly} and \emph{flexibly} derived signatures, which will be formally defined below. Before that, let us introduce the predicate $rigid(\cS)$ to indicate that $\cS$ is either primitive or rigidly derived, and define $ext(\cS^E,\cS):= derived(\cS^E,\cS)~\wedge~S \leq S^E$, which expresses that a signature $S^E$ is an \emph{extension} of signature $S$.

\begin{definition}[Rigidly Derived Signature]
A \emph{rigidly derived signature} is a signature $\cS^D=\langle\{\cC^D_\alpha\}_{\alpha\in\tau},\cP^D\rangle$, where $\cP^D \subseteq \cP$, and each element of $\cC^D_\alpha$ belongs to $Closed(\cL_{Gr}(\cS))$ for some signature $\cS=\langle\cC,\cP\rangle$ such that $rigid(\cS)$. That is, the connectives of $\cS^D$ are closed terms of the grounded language over $\cS$ and its parameters are simply a subset of those of $\cS$, provided that $\cS$ is a rigid signature. 
Notice that the connectives so generated will get the same interpretation in all models; in other words, they are rigidly interpreted (which justifies the chosen wording). {Rigidly} derived signatures thus behave similarly to primitive ones. 
\end{definition}
\begin{definition}[Flexibly Derived Signature]
Similarly, a \emph{flexibly derived signature} is a signature $\cS^D=\langle\{\cC^D_\alpha\}_{\alpha\in\tau},\cP^D\rangle$, where $\cP^D \subseteq \cP$, and each element of $\cC^D_\alpha$ belongs to $Closed(\cL(\cS))$ for some signature $\cS=\langle\cC,\cP\rangle$. That is, the connectives of $\cS^D$ are closed terms of the language over $\cS$ (without further restrictions). Note that, since parameters are also closed terms of the language $\cL(\cS)$, the second component ($\cP^D$) could actually be integrated into the first. Also notice that the connectives so generated, in contrast to those of a primitive or a rigidly derived signature, may get different interpretations in different models, which justifies the qualifications `flexibly' or `flexible'. 
\end{definition}

\begin{proposition} \label{propSig}
	Let $\cS$, $\cS^D$ be signatures such that $derived(\cS^D,\cS)$. We have:
	\begin{enumerate}
		\item \label{prop1} $\cL(\cS^D)\leq\cL(\cS)$ and $\cL(\cS \cup \cS^D)=\cL(\cS)$.
		 
		\item If $\cS \leq \cS^D$, i.e., $ext(\cS^D,\cS)$, then $\cL(\cS^D)=\cL(\cS)$. In words: the extension of a signature still generates the same language.

		\item \label{prop4} If $\cS \nleq \cS^D$, i.e., $not~ext(\cS^D,\cS)$, then $\cL(\cS^D)<\cL(\cS)$. In words: a signature which derives from, but does not extend, another one generates a strict fragment (of the `original' language).
		
		\item \label{prop5} If $\cS^* \leq \cS^D$ then $derived(\cS^*,\cS)$. In words: subsets of derived signatures are also derived (from the same `original' signature).
	\end{enumerate}
\end{proposition}
\begin{proof}
	The proofs are straightforward and left to the reader.
\end{proof}

\begin{example}[STT signature as rigidly derived]
	As noted by \textcite{quine1956unification} and \textcite{henkin1963theory}, the signature $\cS_{\text{STT}}$ can, in principle, be rigidly derived from the signature $\cS_{\text{STT$^=$}}$ (both introduced in Ex.~\ref{example-sig-fttstt1} and Ex.~\ref{example-sig-fttstt2}, respectively). This implies that each connective of $\cS_{\text{STT}}$ can be mapped to a term of $Closed(\cL_{Gr}(\cS_{\text{STT$^=$}}))$. Also note, from Prop.~\ref{propSig}\eqref{prop4}  above, that $\cL(\cS_{\text{STT}})$ can be seen as a strict fragment of $\cL(\cS_{\text{STT$^=$}})$.
\end{example}

\begin{example}[Relational and neighborhood signatures as rigidly derived]
	\label{example-sig-rigid-stt}
	Note that $\cS_{\text{FORS}}$ and $\cS_{\text{SONS}}$ (introduced in Ex.~\ref{example-sig-relstr} and \ref{example-sig-nbhdstr} for relational and neighborhood structures respectively) are proper subsets of the signature $\cS_{\text{STT}}$. Thus, by Prop.~\ref{propSig}\eqref{prop5}, they can also be derived (\emph{rigidly} indeed) from the minimal signature $\cS_{\text{STT$^=$}}$. It follows, from Prop.~\ref{propSig}\eqref{prop4}, that their corresponding languages are strict fragments of $\cL(\cS_{\text{STT}})$, and thus also of $\cL(\cS_{\text{STT$^=$}}$). 
\end{example}

In the remainder of this section we introduce several examples illustrating how several non-classical logics (particularly relevant to normative reasoning) can be given as fragments of the HO language $\cL(\cS_{\text{STT}})$.  We use $wo$ (again) as shorthand for type $w\typa o$, which can be understood as the type of characteristic functions associated with truth sets (sets of possible worlds).
Recalling Ex.~\ref{example-sig-rigid-stt}, we can see that each of the \textit{derived} signatures introduced below is a strict fragment of $\cL(\cS_{\text{STT}})$, and therefore also of $\cL(\cS_{\text{STT$^=$}}$).
First,  we introduce the following abbreviations; cf.~\textcite{J23}:

\begin{definition}[Boolean operators as ($w$-type-lifted) STT terms] \label{definitions-boolean}

$\dot\land$ := $\lambda \phi.\lambda \psi. \lambda w. (\phi~w)\land (\psi~w)$ \qquad $\dot\impl$ := $\lambda \phi.\lambda \psi. \lambda w. (\phi~w)\impl(\psi~w)$

$\dot\lor$  := $\lambda \phi.\lambda \psi. \lambda w. (\phi~w)\lor  (\psi~w)$ \qquad $\dot\neg$ \hskip.35em := $\lambda \phi. \lambda w. \neg (\phi~w)$

\end{definition}

\begin{example}[Rigidly derived signature for S5U]
	\label{example-sig-S5U} We add the following abbreviation to those from  Def.~\ref{definitions-boolean}: $\dot\boxdot^u := \lambda \phi. \lambda w. \forall v. (\phi~v)$.
	A signature for modal logic S5 with universal modality is given by $\cS_{\text{S5U}}=\langle \{\{\dot\neg,\dot\boxdot^u\}_{wo\typa wo},\allowbreak\{\dot\land,\dot\lor,\dot\impl\}_{wo\typa wo\typa wo}\},\{\{p^k\}^{k\in\Nat}_{wo}\} \rangle$, which has been \emph{rigidly} derived from the signature $\cS_{\text{FORS}}$ for relational structures introduced in Ex.~\ref{example-sig-relstr}.
\end{example}

In the remainder each $R^i$ represents an uninterpreted accessibility relation. Similarly $N_1^i$ and $N_2^i$ represent unary and binary neighborhood functions resp.

\begin{example}[Flexibly derived signature for MLK]
	\label{example-sig-MLK}
	We introduce the abbreviations $\dot\boxdot^a := \lambda \phi. \lambda w. \forall v. (R^1~w)~v\impl(\phi~v)$, $\dot\boxdot^p := \lambda \phi. \lambda w. \forall v. (R^2~w)~v\impl(\phi~v)$, $\dot\Diamond^a := \lambda \phi. \dot\neg\,\dot\boxdot^a\,\dot\neg\phi$, and $\dot\Diamond^p := \lambda \phi. \dot\neg\,\dot\boxdot^p\,\dot\neg\phi$.	
A signature for the (bimodal) normal modal logic (extending $K$) is given by  $\cS_{\text{MLK}}=\langle \{\{\dot\neg,\dot\boxdot^a,\dot\Diamond^a,\dot\boxdot^p,\dot\Diamond^p\}_{wo\typa wo},\allowbreak\{\dot\land,\dot\lor,\dot\impl\}_{wo\typa wo\typa wo}\},\{\{p^k\}^{k\in\Nat}_{wo}\} \rangle$, which is \emph{flexibly} derived from the signature for relational structures $\cS_{\text{FORS}}$ introduced in Ex.~\ref{example-sig-relstr}.
\end{example}

\begin{example}[Flexibly derived signature for MDL]
	\label{example-sig-MDL}
	We use here the abbreviation: $\dot O := \lambda \phi. N^1_1~\phi$.
	Now consider the following signature for the family of non-normal monadic deontic logics (MDL) based on a minimal semantics \parencite{chellas1980modal}: $\cS_{\text{SDL}}=\langle \{\{\dot\neg,\dot O\}_{wo\typa wo},\{\dot\land,\dot\lor,\dot\impl\}_{wo\typa wo\typa wo}\},\{\{p^k\}^{k\in\Nat}_{wo}\} \rangle$, which is \emph{flexibly} derived from the signature for neighborhood structures $\cS_{\text{SONS}}$ introduced in Ex.~\ref{example-sig-nbhdstr}.
\end{example}

\begin{example}[Flexibly derived signature for DDL]
	\label{example-sig-DDL}
	We use here the abbreviations: $\dot O^d := \lambda \psi. \lambda \phi. (N^1_2~\phi)~\psi$, $\dot O^a := \lambda \phi. \lambda w. ((N^1_2~(R^1~w))~\phi)~w$, and $\dot O^p := \lambda \phi. \lambda w. ((N^1_2~(R^2~w))~\phi)~w$.	
	Consider the following signature featuring both monadic and dyadic deontic operators: $\cS_{\text{DDL}}=\langle \{\{\dot\neg,\dot O^a,\dot O^b\}_{wo\typa wo},\allowbreak\{\dot\land,\dot\lor,\dot\impl,\dot O^d\}_{wo\typa wo\typa wo}\},\{\{p^k\}^{k\in\Nat}_{wo}\} \rangle$, which has been \emph{flexibly} derived from the signature ${\cS_{\text{FORS}}\cup\cS_{\text{SONS}}}$. Observe that the signature for the DDL by \textcite{Carmo2002}, see also \textcite{C71},  corresponds to $\cS_{\text{CJDDL}}=\cS_{\text{DDL}}\cup\cS_{\text{S5U}}\cup\cS_{\text{MLK}}$, with the latter two as introduced in Ex.~\ref{example-sig-S5U} and \ref{example-sig-MLK} respectively.
\end{example}

\begin{remark}Notice that the (binary) functions $N_2^i$ (having as type: $wo\typa wo\typa wo$) are different from (actually, a generalization of) the neighborhood functions commonly used in dyadic deontic logics (DDL). For instance the function $ob(\phi,\psi):\pow(W)\typa \pow(\pow(W))$, as introduced by e.g.~\textcite{Carmo2002}, would correspond to the type ${wo\typa wo \typa o}$. Recall the discussion in \ref{remark-neighborhood} in Section~\ref{sec:notions:ho-sig}. Our approach has the advantage of facilitating a comparison with algebraic semantics for modal logics (where unary resp.~binary operators are associated with types ${wo\typa wo}$ resp.~${wo\typa wo\typa wo}$).
\end{remark}

\begin{example}[Flexibly derived signature for LFI]
	\label{example-sig-LFI}
	We use here the abbreviations: $\dot\neg^p := \lambda \phi.~\phi~\dot\impl~ (N^1_1~\phi)$, and $\dot\circ := \lambda \phi.~\dot\neg(\phi~\dot\land~(N^1_1~\phi))~\dot\land~(N^2_1~\phi)$.	
	Consider the following signature for the (paraconsistent) logics of formal inconsistency (LFI) with replacement,\footnote{The class LFI of paraconsistent logics was introduced by~\textcite{carnielli2002taxonomy}. They feature a non-explosive negation $\neg$, as well as a (primitive or derived) consistency connective $\circ$ which allows to recover the law of explosion in a controlled way \parencite{carnielli2016paraconsistent}. It has been recently shown that some logics in the hierarchy of LFIs (starting with the minimal logic $mbC$) can be enriched with replacement, and thus given algebraic and neighborhood semantics \parencite{CCF2020}. 
	} based on a neighborhood semantics:
	$\cS_{\text{LFI}}=\langle \{\{\dot\neg^p,\dot\circ\}_{wo\typa wo},\{\dot\land,\dot\lor,\dot\impl\}_{wo\typa wo\typa wo}\},\{\{p^k\}^{k\in\Nat}_{wo}\} \rangle$, which is \emph{flexibly} derived from the signature for neighborhood structures $\cS_{\text{SONS}}$ introduced in Ex.~\ref{example-sig-nbhdstr}.
\end{example}

\begin{remark}[On logical vs. extralogical expressions]
Note that some of the logical connectives of a flexibly derived signature are given varying interpretations in different models. This is because they are articulated by employing parameters (functional constants) such as, e.g., accessibility relations or neighborhood functions. While this phenomenon is well-known in modal logic, our logico-pluralist approach readily exploits (and generalizes) it by viewing the logic of formalization as an additional degree of freedom in the process of logical analysis of argumentative discourse, and thus blurring the distinction between logical/extralogical (resp. syncategorematic/categorematic) expressions.
\end{remark}

Moreover, we can leverage the conceptual tool of \textit{derived signatures} to define F-languages in much the same way as presented in Def.~\ref{def-f-lang}. When the interpretation of their connectives corresponds to that for some particular target logic, such a \textit{derived} F-language can be said to be a \textit{shallow semantical embedding} (SSE). We can then define a corresponding predicate in the host language modeling \textit{validity} for terms of the embedded logic.
	
\begin{definition}[SSEs as language fragments]
	\label{def-sse-der-lang}
	Given a (derived) signature $\cS^T$ (target), such that $derived(\cS^T,\cS^H)$ for some F-signature $\cS^H$ (host), and an interpretation of the connectives of $\cS^T$ as the connectives of some logic $L^T$, we say that $\cL^F(\cS^T)$ is a SSE for logic $L^T$ into $\cL^F(\cS^H)$. We can then define a term $vld(\cdot) \in \cL^F(\cS^H)$, such that $vld(\varphi)$ iff $\varphi \in \cL^F(\cS^T)$ is to be considered as logically valid or true.
\end{definition}

\begin{example}[SSE of MLK into STT]
	\label{example-sse-der-lang}
	The (derived) F-language $\cL^F(\cS_{\text{MLK}})$ generated from the \textit{derived} signature $\cS_{\text{MLK}}$ (cf.~Ex.~\ref{example-sig-MLK})\footnote{
		Recall that $derived(\cS_{\text{MLK}},\cS_{\text{FORS}})$ and $\cS_{\text{FORS}} < \cS_{\text{STT}}$, i.e.~$\cL^F(\cS_{\text{MLK}}) < \cL^F(\cS_{\text{STT}})$.
	} can be said to be a SSE of modal logic K into STT. We define the predicate $vld(\cdot) := \lambda \phi. (\Pi^w\,\phi)$.
\end{example}

It is worth mentioning that we can obtain \textit{faithful} SSEs for canonical extensions of modal logic K (systems KT, KB, S4, etc.) by adding the corresponding restrictions to the accessibility relation as further STT axioms. The notion of \textit{faithfulness} is discussed in the following section.

\subsection{SSEs as Translations -- Faithfulness}
\label{sec:sse-char:sse}

We just saw one variant of the notion of SSEs, namely as the careful selection of a fragment of a HO logical language, in our case STT, to host a target logic. We conceive of this fragment as corresponding to some desired target logical system. In this section, we will characterize another, related notion of a \textit{shallow semantical embedding} (SSE). Starting with an existing formal logical system $L^T$, encoded using a propositional or functional HO language (P- or F-language, cf.~Section~\ref{sec:notions}) we will be able to systematically translate $L^T$ into a corresponding fragment in a host F-language (STT), in a way that preserves logical validity and consequence (i.e.~in a \textit{faithful} manner).

We start by building some useful conceptual background. Notice that the notions of \textit{arities} (as natural numbers) and functional \textit{types} (as recursive structures) are quite similar in a sense: both are enumerable sets of well-formed expressions (which can be defined inductively), and both serve as `labels' for the connectives of a signature. We abstract from this and introduce the notion of \emph{type domain} as an inductively defined set whose elements will be called \emph{types}.

\begin{definition}[Type mapping]
	\label{def-type-mapping}
	Given two type domains $t_1$ and $t_2$, a type mapping is an function $|\cdot|:t_1\rightarrow t_2$. Recalling the (functional) type domain $\tau$ as introduced in Def.~\ref{def-functional-type} and observing that we can take $\Nat$ as a type domain (for \textit{arities}), we define some particularly useful type mappings:
	\begin{itemize}
		\item The mapping from arities to (functional) types (\textit{arity-to-type}) $|\cdot|^\sigma:\Nat\rightarrow\tau$ (with $\sigma\in\tau$, where $|0|^\sigma=\sigma$ and $|k|^\sigma=\sigma^k\typa\sigma$ (with $k\in\Nat$).
		\item The (functional) \textit{type-substitution} mapping $|\cdot|^{\beta/\alpha}:\tau\rightarrow\tau$ (with $\alpha,\beta\in\tau$), where $|\phi|^{\beta/\alpha}$ is obtained by replacing all occurrences of $\alpha$ in $\phi$ by $\beta$.
		\item The \emph{type-lifting} mapping is a special case of the latter: $|\cdot|^{\omega\typa\sigma} := |\cdot|^{\omega\typa\sigma/\sigma}$.
	\end{itemize}	
\end{definition}

We introduce a more general notion of signature encompassing both P- and F-signatures (and also FO signatures if defined appropriately): An \textit{(abstract) signature} over type domain $t$ is just a type-indexed collection of sets of symbols $\cS^t=\{S_\alpha\}_{\alpha\in t}$.\footnote{At this level of abstraction we have no need to differentiate between connectives and parameters. F-signatures can directly be given as (abstract) signatures. For P-signatures we note that every propositional constant $p\in\cP$ can be seen as having arity (i.e., type) zero.}
Type mappings can be naturally extended to structure-preserving mappings between components of signatures, which we term \emph{signature morphisms}.\footnote{The notion presented here is a variant of that introduced by \textcite[Ch.~7]{carnielli2008analysis}. To make this exposition more self-contained and simple we have sidestepped the use of notions from categorial logic \parencite{handbook-logic-cs5} and the theory of institutions \parencite{goguen1992institutions}.}
	
\begin{definition}[Signature morphism]
	\label{def-signature-morphism}
	Let $\cS^{t_1}=\{S^{t_1}_\alpha\}_{\alpha\in t_1}$ and $\cS^{t_2}=\{S^{t_2}_\alpha\}_{\alpha\in t_2}$ be signatures over type domains ${t_1}$ and ${t_2}$; and let ${|\cdot|}:t_1\rightarrow t_2$ be a corresponding type mapping. A \emph{signature morphism} $|\cdot|:\cS^{t_1}\rightarrow\cS^{t_2}$ is a family of maps $h_\alpha: S^{t_1}_\alpha \rightarrow S^{t_2}_{|\alpha|}$, for every $\alpha\in t_1$. A signature morphism is always total but in general not surjective (i.e. an `embedding').
\end{definition}

 Signature morphisms induce syntactical `translation' mappings between languages. $||\cdot||: \cL(\cS^{t_1}) \typa \cL(\cS^{t_2})$. However, translations \textit{proper} involve further restrictions.\footnote{
 	See \textcite{carnielli2009new}, and references therein, for a more thorough discussion of the notion of translations.
 }

\begin{definition}[(Conservative) translation]
	Let $\turns_{L1}$ and $\turns_{L2}$ be the consequence relations associated with two logical systems with languages $L_1$ and $L_2$ respectively, and $P_1, P_2, \dots, P_n$ and $C$ be arbitrary formulas from $L_1$, the mapping $||\cdot||: L_1 \typa L_2$ is a (conservative) translation when:
	\[||P_1||,||P_2||,\dots,||P_n||\,\,\turns_{L2}\,\,||C|| \text{\ \ \ if (and only if)\ \ } P_1,P_2,\dots,P_n\,\,\turns_{L1}\,\,C \]
\end{definition}

\begin{definition}[(Faithful) SSEs as (conservative) translations]
	\label{def:SSE}
	Given a target language $\cL(\cS^{t})$ for type domain $t$ and a host (F-)language $\cL^{(F)}(\cS^{\tau})$, let us define the corresponding type mapping $|\cdot|:t\rightarrow\tau$ and extend it to a signature morphism $|\cdot|:\cS^{t}\rightarrow\Sigma^{\tau}$, where $derived(\Sigma^{\tau},\cS^{\tau})$, thus inducing a translation mapping $||\cdot||: \cL(\cS^{t}) \typa \cL(\cS^{\tau})$ (remember that $\cL(\Sigma^{\tau})\leq \cL(\cS^{\tau})$ by Prop.~\ref{propSig}\eqref{prop1}). This signature morphism together with some special meta-logical validity (or truth) predicate $vld(\cdot) \in \cL(\cS^{\tau})$, is called a (faithful) \emph{shallow semantical embedding} of $\cL(\cS^{t})$ into the HO logic with language $\cL(\cS^{\tau})$, when it induces a (conservative) translation in the manner shown below:
	\[vld(||P_1||),\dots,vld(||P_n ||)\,\turns_{\text{HOL}}\,vld(||C||) \text{\ \ if (and only if)\ \ } P_1,\dots,P_n\,\,\turns_{\text{TL}}\,C \]
		
	$\turns_{\text{TL}}$ and $\turns_{\text{HOL}}$ correspond to the consequence relations for the target logic and the host (HO) logic resp. Observe that the definition above corresponds to the notion of \textit{global} validity. In special cases (e.g. modal logics) we can exploit the deduction metatheorem to define a (faithful) SSE for \textit{local} validity as:
	\[\turns_{\text{HOL}}\,vld(||P_1 \impl\dots\impl P_n \impl C||) \text{\ \ if (and only if)\ \ } P_1,\dots,P_n\,\,\turns_{\text{TL(local)}}\,C \]
\end{definition}

For the following illustrative example we reuse the abbreviations defined in Def.~\ref{definitions-boolean} and append the following:
$\dot\boxdot := \lambda \phi. \lambda w. \forall v. (R^1~w)~v\impl(\phi~v)$.

\begin{example}[SSE of ML into STT]
	\label{example-sse-ml-stt}
	Given the following P-signature for modal propositional
	logic (recall also Ex.~\ref{example-sig-modal}): 
	$$\cS_{\text{ML}}=\langle \{\{\neg,\Box\}_1,\{\land,\lor,\impl\}_2\},\{p^k\}^{k\in\Nat}_0 \rangle.$$
	Now consider the following F-signature flexibly derived from signature $\cS_{\text{FORS}}$ (Ex.~\ref{example-sig-relstr}) for the language of relational structures:	
	$$\cS_{\text{FORS*}}=\langle \{\{\dot\neg,\dot\boxdot\}_{wo\typa wo},\{\dot\land,\dot\lor\}_{wo\typa wo\typa wo}\},\{\{p^k\}^{k\in\Nat}_{wo}, \{R^1\}_{w\typa w\typa o}\} \rangle$$
	and the signature morphism induced by \textit{arity-to-type} mapping ${|\cdot|}^{wo}:\Nat\typa wo$ (Def~\ref{def-type-mapping}), such that
	
	${|\,p^k\,|}^{wo}=p^k$,\qquad\qquad
	${|\,\neg\,|}^{wo}=\dot\neg$,\qquad\qquad
	${|\,\Box\,|}^{wo}~=\dot\boxdot$,\qquad\qquad
	
	${|\land|}^{wo}\,=\dot\land$,\qquad\qquad
	${|\lor|}^{wo}~=\dot\lor$,\qquad\qquad 
	${|\impl|}^{wo}=\lambda \phi. \lambda \psi.\, \dot\neg\phi~\dot\lor~\psi$.
	\ \\

	This signature morphism together with the $\cL(\cS_{\text{FORS}})$ term\footnote{
		Observe that according to the definition of SSE (Def.~\ref{def:SSE}) the term $vld(\cdot)$ may feature additional connectives ($\Pi^w$ in this case) not part of $\cS_{\text{FORS*}}$, while still part of $\cS_{\text{FORS}}$. 
	}
	$$vld(\cdot) := \lambda \phi. (\Pi^w\,\phi)$$
	defines a SSE of $\cL(\cS_{\text{ML}})$ into the relational language $\cL(\cS_{\text{FORS}}) \leq \cL(\cS_{\text{STT}})$. Hence we obtain a \textit{faithful} SSE of the language of (mono-)modal logic into STT. The faithfulness of (a multi-modal generalization of) this SSE has been proved in \textcite{J21}.
\end{example}

\begin{remark}
	Compare Ex.~\ref{example-sse-ml-stt} with Ex.~\ref{example-sse-der-lang}. The end result is, \textit{mutatis mutandis}, similar in both. Observe that, in the case of SSEs as `plain' language fragments (Def.~\ref{def-sse-der-lang}) the question of faithfulness may not appear if the target logic becomes \textit{defined} as the corresponding fragment of the host F-language (STT). This can happen in the case of logic combinations, which quite often do not correspond to any existent (i.e.~previously studied) logical system.
\end{remark}

\subsection{SSEs of logic combinations}
\label{sec:sse-char:log-comb}
We have previously explored two related characterizations of the notion of \textit{shallow semantical embeddings} (SSE); the first as language fragments (Def.~\ref{def-sse-der-lang}), and the second as translations (Def.~\ref{def:SSE}). We saw that the  second characterization allows us to articulate a notion of faithfulness in a very precise manner. Yet SSEs of logical systems and their combinations, employing the first notion, may often involve a tacit faithfulness claim. Let us consider the following:

\begin{example}[SSE of QML into STT]
	\label{example-sse-QML}
	We can obtain a signature for quantified modal logic (QML) by `merging' signatures (analogously to what we showed in Ex.~\ref{example-sig-DDL}). In this case the signatures would correspond to (world-lifted) FO logic (Ex.~\ref{example-sig-fttfol}) and modal logic K (Ex.~\ref{example-sig-MLK}). A SSE of QML (in the first sense) would thus amount to: $\cL^F(\cS_{\text{MLK}} \cup |\cS_{\text{FOL-F}}|^{w\typa o / o})$, where $|\cdot|^{w\typa o / o}$ stands for the signature morphism extending the corresponding type-lifting mapping (recall Defs.~\ref{def-type-mapping} and~\ref{def-signature-morphism}). We have $vld(\cdot) := \lambda \phi. (\Pi^w\,\phi)$ as before. 
\end{example}

An SSE for a quantified modal logic (extending QML in Ex.~\ref{example-sse-QML}) has been proved faithful by \textcite{J23}.

We can also define, in an analogous manner, an SSE for the DDL by \textcite{Carmo2002}, whose signature we introduced in Ex.~\ref{example-sig-DDL}. We note that the corresponding semantical constraints on $N^1_2$, $R^1$ and $R^2$, when encoded in the host language, $\cL^F(\cS_{\text{STT}})$, give rise to a faithful SSE \parencite{C71}. Among several others, faithful SSEs also exist for quantified conditional logics \parencite{J31}, as well as for I/O logic \parencite{J46}, and free logic \parencite{J40}.

It is worth highlighting that the conceptual framework introduced in this section provides a more intuitive grasp of the relationship between SSE and other approaches to combining logics, such as \textit{algebraic fibring}, a generalization of \textit{fusion} of modal logics, cf.~\textcite{sep-logic-combining}. Combining \textit{propositional} logics by `merging signatures', as illustrated in Ex.~\ref{example-sig-DDL} for the DDL by~\textcite{Carmo2002}, can indeed be seen as an instance of the general notion of algebraic fibring \parencite{carnielli2008analysis}. Theoretical connections between algebraic fibring and SSE of combinations of logics are still ongoing work and remain largely unexplored.

%%% Local Variables: 
%%% TeX-master: "root.tex"
%%% End: 

\section{Encoding Formal Argumentation}
\label{sec:arg}

\subsection{Notions of Formal Argumentation}
\label{sec:arg:notions}

We follow on the footsteps of previous work depicting the logical analysis of argumentative discourse as a hermeneutical process \parencite{CH,B19}, whereby informal arguments become initially reconstructed as formalized enthymemes and, after that, incrementally evolve towards purely formal deductions satisfying further inferential adequacy criteria (including, among others, consistency and minimality). We thus conceive of a formal argument (quite generally) as a labeled pair consisting of (i) an identifier, (ii) a set of formulas (premises), together with (iii) a formula (conclusion); where both (ii) and (iii) are articulated using a certain logic for formalization. We reserve the adjective \textit{deductive} for those arguments in which the premises logically entail the conclusion.

\begin{definition}[(Deductive) Argument]
\label{definition1}
An argument is a labeled ordered pair $A:\langle\varphi,\alpha\rangle$, where $A$ is an identifier (label) and $\varphi \cup \{\alpha\}$ is a set of formulas of the language of some underlying logic $L$. An argument $A$ is said to be \textit{deductive} (modulo logic $L$) if additionally $\varphi\vdash_{L}\alpha$. The formulas in $\varphi$ are the premises (aka.~assumptions or support) of the argument,
and the formula $\alpha$ is the conclusion (or claim) of the argument.
\end{definition}

Other constraints we may \textit{optionally} set on arguments are consistency: $\varphi$ has to be logically consistent
(according to the chosen logic $L$); and minimality: there is no $\psi \subset \varphi$ such that $\psi\vdash_L\alpha$.
For an argument $A:\langle\varphi,\alpha\rangle$ the function \emph{Premises(A)} returns $\varphi$ (in a specific order) and
\emph{Conclusion(A)} returns $\{\alpha\}$.

Observe that this definition is more permissive than others in the literature; cf.~in particular \textit{deductive argumentation} \parencite{BH}. The `lifecycle' of an argument can be followed with the aid of its identifier. Every pair $\langle\varphi,\alpha\rangle$ can be seen as the current `snapshot' of a certain argument during the give-and-take, iterative process of formal reconstruction. Hence this `evolving argument' will, now and again, become logically valid, consistent, or minimal depending on its current state and, quite importantly, on the current formalization logic. In line with the SSE approach, this (object) logic amounts to the currently chosen fragment of STT used to encode argument's sentences (recall the discussion in Section~\ref{sec:sse-char}).

Similarly to other approaches towards structured argumentation \parencite{baroni2018handbook}, different kinds of \emph{attack} relations between arguments (defeaters, undercuts, rebuttals, etc.) can be easily introduced and interrelated in the meta-logic. However, in the present work we limit ourselves to a quite general notion: an attack between (a set of) argument(s) \emph{A} and argument \emph{B} corresponds to the inconsistency (modulo the meta-logic, e.g.~STT) of the set of formulas formed by the conclusion(s) of \emph{A} together with the premises of \emph{B}. We also consider \emph{support} relations between arguments in an analogous manner, where the conclusion(s) of (a set of) argument(s) \emph{A} logically entails \emph{B} in the meta-logic.\footnote{
	Our notion of argument support is different, though not unrelated, to other notions in the literature; cf.~\textcite{BAF,prakkenmodelling}; cf.~also \textcite{cohen2014survey,baroni2018handbook} for a more comprehensive overview.
} Notice that our framework allows us to extend the notions of attack (support) to the case where two (or more) arguments \textit{jointly} attack (support) another.

\begin{definition}[Attack]
\label{definition2}
An argument \emph{A} \textit{attacks} argument \emph{B} iff the set \textit{Conclusion(A)}~$\cup$~\textit{Premises(B)}
is inconsistent. This definition can be seamlessly extended to two (or more) arguments: $A_1$ and $A_2$ (jointly) attack $B$ iff the set \textit{Conclusion($A_1$)}~$\cup$~\textit{Conclusion($A_2$)}~$\cup$~\textit{Premises(B)} is inconsistent. 
\end{definition}

Notice that this definition subsumes the more traditional one for classical logic,
\textit{Conclusion(A)}~${\vdash\neg\bigwedge X}$ for some $X\subseteq\,$\textit{Premises(B)}, while allowing for paraconsistent
formalization logics
where explosion (inconsistency) does not necessarily follow from pairs of contradictory formulas (recall Ex.~\ref{example-sig-LFI}).

\begin{definition}[Support]
\label{definition3}
An argument \emph{A} \textit{supports} argument \emph{B} iff \textit{Conclusion(A)}~${\vdash X}$
for some $X\in$~\textit{Premises(B)}. This definition can be seamlessly extended to two (or more) arguments:
$A_1$ and $A_2$ (jointly) support $B$ iff \textit{Conclusion($A_1$)}~$\cup$~\textit{Conclusion($A_2$)}~$\vdash X$
for some $X\in$~\textit{Premises(B)}.
\end{definition}

At this point, we would like to highlight the similarity in spirit between ours and the ``descriptive approach'' \parencite{BH} towards reconstructing argument graphs from natural-language sources, in which we carry out a reconstruction process taking as input an abstract argument graph, together with some informal text description of each argument. The task thus becomes one of finding the appropriate logical formulas for encoding the premises and conclusion of each argument, according to the choice of the logic of formalization. As will be illustrated in the case study in Section~\ref{sec:case-study}, there is often a need for finding appropriate `implicit' premises which render the individual arguments logically valid and additionally honor their intended dialectical role in the input abstract graph (i.e., attacking or supporting other arguments). This interpretive aspect, in particular, has been emphasized in our \textit{computational hermeneutics} approach \parencite{CH,B19}, as well as the possibility of modifying (or reconstructing) the input abstract argument graph as new insights, resulting from the formalization process, appear. In their exposition of structured, deductive argumentation, \textcite{BH} duly highlight these aspects and, quite important for us, they emphasize the fact that ``richer'' logic formalisms (i.e., more expressive than ``rule-based'' ones) appear more appropriate for reconstructing ``real-world arguments'' in spite of their higher computational complexity. Such representational and interpretive issues are tackled in our approach by the use of different (combinations of) non-classical and higher-order logics for formalization. For this we utilize the shallow semantical embeddings (SSE) approach as introduced in the previous sections.

\subsection{Isabelle/HOL encoding}
\label{sec:arg:isabelle}

The claim that an argument, or a set of arguments, attacks or supports another argument is, in our approach, conceived as (the claim of) an argument in itself. This argument needs to be reconstructed as logically valid \textit{in the meta-logic}, eventually by adding implicit (meta-logical) premises.
As an illustration, we introduce a term encoded in the language of Church's \textit{simple type theory} STT (cf.~Ex.~\ref{example-sig-fttstt1} and Ex.~\ref{example-stt-lang}) which models the notion of \emph{attack} discussed above:
$$attacks_1 := \lambda\varphi . \lambda\psi .~(\varphi~\land~\psi)~\impl~F$$

The type of the term $attacks_1$ above is $o \typa o\typa o$, corresponding to a binary predicate. In a similar spirit we can introduce other dialectical relations.
We illustrate below the corresponding definitions employing \textit{Isabelle/HOL} syntax. Notice that the symbols are slightly different. We use \textit{Isabelle}'s keyword \isakeyword{abbreviation} to introduce these definitions as syntactic shorthand.

\begin{isabellebody}\normalsize\isanewline%
\ \ \ \ %
\isacommand{abbreviation}%
\ %
{\isachardoublequoteopen}attacks{\isadigit{1}}\ {\isasymphi}\ {\isasympsi}\ \ \ \ {\isasymequiv}\ {\isacharparenleft}{\isasymphi}\ {\isasymand}\ {\isasympsi}{\isacharparenright}\ {\isasymlongrightarrow}\ False{\isachardoublequoteclose}\ \ \ \ \ \ %
\isamarkupcmt{one attacker}\isanewline
\ \ \ \ %
\isacommand{abbreviation}%
\ %
{\isachardoublequoteopen}supports{\isadigit{1}}\ {\isasymphi}\ {\isasympsi}\ \ \ {\isasymequiv}\ {\isasymphi}\ {\isasymlongrightarrow}\ {\isasympsi}{\isachardoublequoteclose}\ \ \ \ \ \ \ \ \ \ \ \ \ \ \ \ \ \ %
\isamarkupcmt{one supporter}\isanewline
\ \ \ \ %
\isacommand{abbreviation}%
\ 
{\isachardoublequoteopen}attacks{\isadigit{2}}\ {\isasymgamma}\ {\isasymphi}\ {\isasympsi}\ \ {\isasymequiv}\ {\isacharparenleft}{\isasymgamma}\ {\isasymand}\ {\isasymphi}\ {\isasymand}\ {\isasympsi}{\isacharparenright}\ {\isasymlongrightarrow}\ False{\isachardoublequoteclose}\ %
\isamarkupcmt{two attackers}\isanewline
\ \ \ \ %
\isacommand{abbreviation}%
\ %
{\isachardoublequoteopen}supports{\isadigit{2}}\ {\isasymgamma}\ {\isasymphi}\ {\isasympsi}\ {\isasymequiv}\ {\isacharparenleft}{\isasymgamma}\ {\isasymand}\ {\isasymphi}{\isacharparenright}\ {\isasymlongrightarrow}\ {\isasympsi}{\isachardoublequoteclose}\ \ \ \ \ \ \ \ \ \ \ %
\isamarkupcmt{two supporters}%
\end{isabellebody}%

%%% Local Variables: 
%%% TeX-master: "root"
%%% End: 
\section{Case Study: Climate Engineering Debate}
\label{sec:case-study}

\subsection{Background}
\label{sec:case-study:motivation}

Climate Engineering (CE), aka. Geo-engineering, is the intentional large-scale intervention
in the Earth's climate system in order to counter climate change.
Proposed CE technologies (e.g., solar radiation management, carbon dioxide removal)
are highly controversial, spurring global debates about whether
and under which conditions they should be considered.
Criticisms to CE range from diverting attention and resources from much needed mitigation policies
to potentially catastrophic side-effects; thus the cure may become worse than the disease.
The analyzed arguments around the CE debate presented in this paper originate from the book of 
\textcite{CE}, which is a slightly modified and updated translation of a study
commissioned by the German Federal Ministry of Education and Research (BMBF) on
``Ethical Aspects of Climate Engineering'' finalized in spring 2011.\footnote{
	\citeauthor{CE}'s work aimed at providing a quite complete overview of
	the arguments around CE at the time. However, since the CE controversy has been advancing rapidly, it is to expect that their work has partially become outdated meanwhile.
}
The illustrative analysis carried out in the present paper focuses on a small subset of
the CE argumentative landscape, namely on those arguments concerned with the ``ethics of risk''
(\cite{CE} p. 38ff.) which point out (potentially dangerous) uncertainties in future deployment of CE.\footnote{
	\textit{Isabelle/HOL} sources for this case study have been made available online (\url{https://github.com/davfuenmayor/CE-Debate}).
}

\subsection{Individual (Component) Arguments}
\label{sec:case-study:args}

As has been observed by \citeauthor{CE}, incalculable side-effects and imponderables
constitute one of the main reasons against CE technology deployment. Thus, arguments from the
\emph{ethics of risk} primarily support the thesis: ``CE deployment is morally wrong'' and make for an argument cluster with a non-trivial dialectical structure which we aim at
reconstructing in this section.
We focus on six arguments from the ethics of risk,
which entail that the deployment of CE technologies (today as in the future) is not desirable
because of being morally wrong (argument A22). Supporting arguments of A22 are: A45, A46, A47, A48, A49
(using the original notation by \citeauthor{CE}).
In particular, two of these arguments, namely A48 and A49, are further attacked by A50 and A51.\footnote{
	We strive to remain as close as possible to the original argument network as introduced by \textcite{CE} (with one exception concerning the dialectical relation among arguments A47, A48, A50 and A22, which will be commented upon later on). The reader will notice that some of the arguments could have been merged together. However, Betz and Cacean have deliberately decided not to do so.
	We conjecture that this is due to traceability concerns, given the fact that most arguments have been compiled from different bibliographic sources and authors; cf.~\textcite{BAF} and \textcite{prakkenmodelling} for a discussion on this sort of issues.
}

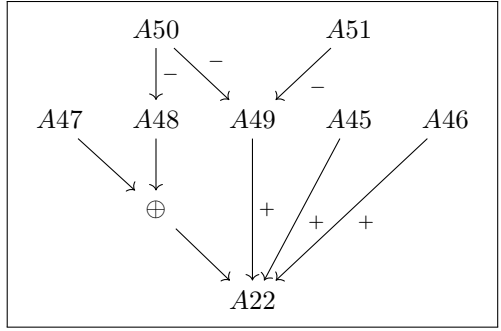
\begin{figure}[tp] \normalsize \centering
	\fbox{
		\begin{tikzcd}[row sep=1em,column sep=1em]
			{}  \& A50 \arrow{d}{-} \arrow{dr}{-} \& {} \& A51 \arrow{dl}{-} \& {} \\[1em]
			A47 \arrow{dr} \& A48 \arrow{d} \& A49 \arrow{dd}{+} \& A45 \arrow{ddl}{+} \& A46 \arrow{ddll}{+} \\[1em]
			{} \& \oplus \arrow{dr} \& {} \& {} \& {} \\[1em]
			{} \& {}\& A22 \& {} \& {}
		\end{tikzcd}
	}
	\caption{Abstract argumentation network for the ethics of risk cluster in the CE debate
  (arrows labeled with +/- indicate \textit{support}/\emph{attack}); $\oplus$ indicates a \emph{joint support}.}
	\label{figArgNetwork}
\end{figure}

\subsubsection{Ethics of Risk Argument (A22)}

\textit{CE deployment is not desirable since it is morally wrong}. This argument has as premise: ``CE deployment is morally wrong'' and as conclusion:
``CE deployment is not desirable''. Notice that both sentences are formalized as (Boolean-valued) propositions. We are thus  restricting ourselves, for the time being,
to a propositional logic. We introduce two new, uninterpreted propositional constants
(``CEisWrong'' and ``CEisNotDesirable'') and interrelate them by means of an implicit premise (``If CE is morally wrong then it is not desirable'').

Since this is the first argument to be represented in the proof assistant \textit{Isabelle/HOL} in this work, we will pay special attention to the syntactic elements used for its formulation in the system. First notice the use of the keyword \isakeyword{consts} to introduce the following non-interpreted constants; their type being $bool$ (corresponding to type $o$ in the exposition in Section~\ref{sec:notions:ho-sig}).

%%%%%%%%%%%%%%%%%%%%%%%%%%%%%%%%%%%%%%%%%%%%%%
\begin{isabellebody}\normalsize\isanewline
\isacommand{consts}
\ CEisWrong{\isacharcolon}{\isacharcolon}{\isachardoublequoteopen}bool{\isachardoublequoteclose}\ %
\isamarkupcmt{uses Boolean type%
}\isanewline
\isacommand{consts}
\ CEisNotDesirable{\isacharcolon}{\isacharcolon}{\isachardoublequoteopen}bool{\isachardoublequoteclose}
\isanewline
\end{isabellebody}%
%%%%%%%%%%%%%%%%%%%%%%%%%%%%%%%%%%%%%%%%%%%%%%%

Below we employ the keyword \isakeyword{definition} to introduce interpreted constants (of Boolean type). The first two definitions introduce an explicit and an implicit premise, labeled \textit{A22-P1} and \textit{A22-P2} respectively, and the third one introduces its conclusion, labeled \textit{A22-C}.\footnote{Notice that we will keep this same notational convention throughout this work: an argument identifier, say \textit{A99}, followed by a hyphen and either a letter $P$ followed by a premise number (say 2), or by a letter $C$. For example \textit{A99-P2} stands for $\text{Premises(A99)}[2]$ and \textit{A99-C} stands for $\text{Conclusion(A99)}$. Recall the discussion in Section~\ref{sec:arg}.}
Definitions introduce an equivalence between two formulas (by employing the symbol {\isasymequiv}) with the \textit{definiendum} on its left-hand side and the \textit{definiens} on its right-hand side.

%%%%%%%%%%%%%%%%%%%%%%%%%%%%%%%%%%%%%%%%%%%%
\begin{isabellebody}\normalsize\isanewline	
\isacommand{definition}%
\ {\isachardoublequoteopen}A{\isadigit{2}}{\isadigit{2}}{\isacharunderscore}P{\isadigit{1}}\ {\isasymequiv}\ CEisWrong{\isachardoublequoteclose}\ \isanewline
\isacommand{definition}%
\ {\isachardoublequoteopen}A{\isadigit{2}}{\isadigit{2}}{\isacharunderscore}P{\isadigit{2}}\ {\isasymequiv}\ CEisWrong\ {\isasymlongrightarrow}\ CEisNotDesirable{\isachardoublequoteclose}\ %
\isamarkupcmt{implicit premise%
}\isanewline
\isacommand{definition}%
\ {\isachardoublequoteopen}A{\isadigit{2}}{\isadigit{2}}{\isacharunderscore}C\ \ {\isasymequiv}\ CEisNotDesirable{\isachardoublequoteclose}%
\isanewline
\end{isabellebody}%
%%%%%%%%%%%%%%%%%%%%%%%%%%%%%%%%%%%%%%%%%%%%%

Observe that the object logic employed to formalize this argument corresponds to a propositional fragment of STT, e.g.,~by employing the F-signature $\cS_{\text{CPL}}$ from Ex.~\ref{example-sig-fttcl}. Since $\cL^F(\cS_{\text{CPL}}) \leq \cL^F(\cS_{\text{STT}})$, this encoding can be considered an SSE of the sort discussed in Def.~\ref{def-sse-der-lang}. The second SSE approach (recalling Section~\ref{sec:sse-char:sse}) would have as starting point a classical propositional logic (using P-signature $\cS_{\text{CPL}}$ from Ex.~\ref{example-sig-cl}) and then extend the \textit{arity-to-type} mapping (Def.~\ref{def-type-mapping}) $|\cdot|^o:\Nat\rightarrow\tau$ into a signature morphism (recall Def.~\ref{def:SSE}). Observe that such an SSE is trivially faithful. Note also that we use the deduction meta-theorem of classical propositional logic to encode $A_1, A_2,\dots, A_n\vdash_{\text{CPL}}B$ as $A_1 \land A_2 \land\dots\land A_n \impl B$.\footnote{
	These considerations are, for sure, an overkill at the present moment, but they duly illustrate the issues at stake and prepare the ground for modal object logics.
}

Below we employ the model finder \emph{Nitpick} \parencite{Nitpick} to find a model (not shown here) satisfying both premises and
conclusion of the formalized argument. This shows consistency.

%%%%%%%%%%%%%%%%%%%%%%%%%%%%%%%%%%%%%%%%%%%%
\begin{isabellebody}\normalsize\isanewline
\isacommand{lemma}
\ {\isachardoublequoteopen}A{\isadigit{2}}{\isadigit{2}}{\isacharunderscore}P{\isadigit{1}}\ {\isasymand}\ A{\isadigit{2}}{\isadigit{2}}{\isacharunderscore}P{\isadigit{2}}\ {\isasymlongrightarrow}\ A{\isadigit{2}}{\isadigit{2}}{\isacharunderscore}C{\isachardoublequoteclose}\ \isacommand{nitpick}
{\isacharbrackleft}satisfy{\isacharbrackright}
\isacommand{oops}
\ %
\isamarkupcmt{model% found: consistent
}%
\isanewline
\end{isabellebody}%
%%%%%%%%%%%%%%%%%%%%%%%%%%%%%%%%%%%%%%%%%%%%%

This first argument (A22) serves as a quite straightforward illustration of the role of implicit,
unstated premises in enabling the reconstruction of a candidate argument as a \textit{deductive} argument. 
In this first example, we utilize \textit{Isabelle}'s simplifier (term rewriting engine) to verify that the conclusion follows from the premises.

%%%%%%%%%%%%%%%%%%%%%%%%%%%%%%%%%%%%%%%%%%%%
\begin{isabellebody}\normalsize\isanewline
\isacommand{theorem}%
\ A{\isadigit{2}}{\isadigit{2}}{\isacharunderscore}valid{\isacharcolon}\ {\isachardoublequoteopen}A{\isadigit{2}}{\isadigit{2}}{\isacharunderscore}P{\isadigit{1}}\ {\isasymand}\ A{\isadigit{2}}{\isadigit{2}}{\isacharunderscore}P{\isadigit{2}}\ {\isasymlongrightarrow}\ A{\isadigit{2}}{\isadigit{2}}{\isacharunderscore}C{\isachardoublequoteclose}\isanewline
\ \ %
\isacommand{unfolding}%
\ A{\isadigit{2}}{\isadigit{2}}{\isacharunderscore}C{\isacharunderscore}def\ A{\isadigit{2}}{\isadigit{2}}{\isacharunderscore}P{\isadigit{1}}{\isacharunderscore}def\ A{\isadigit{2}}{\isadigit{2}}{\isacharunderscore}P{\isadigit{2}}{\isacharunderscore}def\ \isacommand{by}%
\ simp\ %
\isamarkupcmt{proved% by simplifier%
}%
%\isanewline
\end{isabellebody}%
%%%%%%%%%%%%%%%%%%%%%%%%%%%%%%%%%%%%%%%%%%%%%

\subsubsection{Termination Problem (A45)}

\textit{CE measures do not possess viable exit options. If deployment is terminated abruptly,
catastrophic climate change ensues}.\footnote{Cf. \textcite{CE} for sources for these
and other proposed theses and arguments in the CE debate.}
This argument (in the given formulation) corresponds to a sort of causal conditional\footnote{
	As mentioned before, we employ in our analysis a material conditional for simplicity of exposition. Other, more appropriate, conditional operators can be employed. For instance, a faithful SSE of a quantified conditional logic has been presented by \textcite{J31}.
} and can be seen as an enthymeme with both an implicit premise and an implicit conclusion. We add as implicit premise (A45-P1) that there is a real possibility of CE interventions being terminated abruptly. The dreaded conclusion being the possibility of a CE-caused catastrophe (A45-C).

Observe that we are employing a possible-worlds semantics for our logic of formalization. The type of propositional atoms is thus \isa{w{\isasymRightarrow}bool{\isachardoublequoteclose}}, which corresponds to the type for characteristic functions of sets of worlds.

%%%%%%%%%%%%%%%%%%%%%%%%%%%%%%%%%%%%%%%%%%%%
\begin{isabellebody}\normalsize\isanewline
\isacommand{consts}%
\ CEisTerminated{\isacharcolon}{\isacharcolon}{\isachardoublequoteopen}w{\isasymRightarrow}bool{\isachardoublequoteclose}\ \ %
\isamarkupcmt{uses type for (world-lifted) propositions%
}\isanewline
\isacommand{consts}%
\ CEisCatastrophic{\isacharcolon}{\isacharcolon}{\isachardoublequoteopen}w{\isasymRightarrow}bool{\isachardoublequoteclose}\isanewline
\isacommand{definition}%
\ {\isachardoublequoteopen}A{\isadigit{4}}{\isadigit{5}}{\isacharunderscore}P{\isadigit{1}}\ {\isasymequiv}\ \isactrlbold {\isasymdiamond}CEisTerminated{\isachardoublequoteclose}\qquad\qquad\qquad%
\isamarkupcmt{implicit premise}
\isanewline
\isacommand{definition}%
\ {\isachardoublequoteopen}A{\isadigit{4}}{\isadigit{5}}{\isacharunderscore}P{\isadigit{2}}\ {\isasymequiv}\ CEisTerminated\ \isactrlbold {\isasymrightarrow}\ CEisCatastrophic{\isachardoublequoteclose}\isanewline
\isacommand{definition}%
\ {\isachardoublequoteopen}A{\isadigit{4}}{\isadigit{5}}{\isacharunderscore}C\ \ {\isasymequiv}\ \isactrlbold {\isasymdiamond}CEisCatastrophic{\isachardoublequoteclose}\qquad\qquad\qquad%
\isamarkupcmt{implicit conclusion}
\isanewline
\end{isabellebody}%
%%%%%%%%%%%%%%%%%%%%%%%%%%%%%%%%%%%%%%%%%%%%%

Notice that we have introduced in the above formalization the 
modal operator $\Diamond$ to signify that a proposition is possibly true. The logic of formalization for this argument corresponds to a modal logic K, i.e., the modal logic discussed in the Examples~\ref{example-sse-der-lang} and \ref{example-sse-ml-stt}. Hence, this can be seen as an SSE of the sorts discussed in Sections~\ref{sec:sse-char:derived} and~\ref{sec:sse-char:sse}. 

The expression $vld(\varphi)$ in the SSE for modal logics (see Def.~\ref{def-sse-der-lang} and Def.~\ref{def:SSE}) is represented as $[\vdash\varphi]$ in the \textit{Isabelle/HOL} encoding; it corresponds to \textit{global truth} in modal logic, i.e., truth in all worlds, formalized as: $\Pi^w\,\varphi$ (i.e.~$\forall w.\,\varphi~w$). In particular, observe that we employ the notion of \textit{global consequence} for the formalization of this argument in modal logic, and whose definition in \textit{Isabelle/HOL} is also illustrated below.

%%%%%%%%%%%%%%%%%%%%%%%%%%%%%%%%%%%%%%%%%%%%%%%%%%%%%%%
\begin{isabellebody}\normalsize\isanewline	
\isacommand{abbreviation}%
\ valid{\isacharcolon}{\isacharcolon}{\isachardoublequoteopen}wo{\isasymRightarrow}bool{\isachardoublequoteclose}\ \  {\isacharparenleft}{\isachardoublequoteopen}{\isacharbrackleft}{\isasymturnstile}\ {\isacharunderscore}{\isacharbrackright}{\isachardoublequoteclose}{\isacharparenright}\ \isakeyword{where}\isanewline
\ \ \ \  {\isachardoublequoteopen}{\isacharbrackleft}{\isasymturnstile}\ \ {\isasymphi}{\isacharbrackright}\ \ {\isasymequiv}\ {\isasymforall}w{\isachardot}\ {\isasymphi}\ w{\isachardoublequoteclose}
	
\isacommand{abbreviation}%
\ conseq{\isacharunderscore}global{\isacharcolon}{\isacharcolon}{\isachardoublequoteopen}wo{\isasymRightarrow}wo{\isasymRightarrow}bool{\isachardoublequoteclose}\ {\isacharparenleft}{\isachardoublequoteopen}{\isacharbrackleft}{\isacharunderscore}\ {\isasymturnstile}\isactrlsub g\ {\isacharunderscore}{\isacharbrackright}{\isachardoublequoteclose}{\isacharparenright}\ \isakeyword{where}\ \isanewline
\ \ \ \ {\isachardoublequoteopen}{\isacharbrackleft}{\isasymphi}\ {\isasymturnstile}\isactrlsub g\ {\isasympsi}{\isacharbrackright}\ {\isasymequiv}\ {\isacharbrackleft}{\isasymturnstile}\ {\isasymphi}{\isacharbrackright}\ {\isasymlongrightarrow}\ {\isacharbrackleft}{\isasymturnstile}\ {\isasympsi}{\isacharbrackright}{\isachardoublequoteclose}

\isacommand{abbreviation}%
\ conseq{\isacharunderscore}global{\isadigit{2}}{\isacharcolon}{\isacharcolon}{\isachardoublequoteopen}wo{\isasymRightarrow}wo{\isasymRightarrow}wo{\isasymRightarrow}bool{\isachardoublequoteclose}\ {\isacharparenleft}{\isachardoublequoteopen}{\isacharbrackleft}{\isacharunderscore}{\isacharcomma}{\isacharunderscore}\ {\isasymturnstile}\isactrlsub g\ {\isacharunderscore}{\isacharbrackright}{\isachardoublequoteclose}{\isacharparenright}\ \isakeyword{where}\ \isanewline
\ \ \ \ {\isachardoublequoteopen}{\isacharbrackleft}{\isasymphi}{\isacharcomma}\ {\isasymgamma}\ {\isasymturnstile}\isactrlsub g\ {\isasympsi}{\isacharbrackright}\ {\isasymequiv}\ {\isacharbrackleft}{\isasymturnstile}\ {\isasymphi}{\isacharbrackright}\ {\isasymlongrightarrow}\ {\isacharparenleft}{\isacharbrackleft}{\isasymturnstile}\ {\isasymgamma}{\isacharbrackright}\ {\isasymlongrightarrow}\ {\isacharbrackleft}{\isasymturnstile}\ {\isasympsi}{\isacharbrackright}{\isacharparenright}{\isachardoublequoteclose}\isanewline
\end{isabellebody}%

%%%%%%%%%%%%%%%%%%%%%%%%%%%%%%%%%%%%%%%%%%%%%%%%%%%%%%%%
Having defined above the object-logical notion of logical consequence for our SSE (Def.~\ref{def:SSE}), we can now formulate the argument and (i) verify that it is consistent by finding a satisfying model using model finder \textit{Nitpick} \parencite{Nitpick}, and (ii) prove it by using the tableau prover \textit{blast}. Both tools come integrated into \textit{Isabelle/HOL}.

%%%%%%%%%%%%%%%%%%%%%%%%%%%%%%%%%%%%%%%%%%%%
\begin{isabellebody}\normalsize\isanewline
\isacommand{lemma}%
\ {\isachardoublequoteopen}{\isacharbrackleft}A{\isadigit{4}}{\isadigit{5}}{\isacharunderscore}P{\isadigit{1}}{\isacharcomma}\ A{\isadigit{4}}{\isadigit{5}}{\isacharunderscore}P{\isadigit{2}}\ {\isasymturnstile}\isactrlsub g\ A{\isadigit{4}}{\isadigit{5}}{\isacharunderscore}C{\isacharbrackright}{\isachardoublequoteclose}\ \isacommand{nitpick}%
\ {\isacharbrackleft}satisfy{\isacharbrackright}%
\ %
\isacommand{oops}%
\ %
\isamarkupcmt{model found%: consistent%
}%
\isanewline
\isacommand{theorem}%
\ A{\isadigit{4}}{\isadigit{5}}{\isacharunderscore}valid{\isacharcolon}\ {\isachardoublequoteopen}{\isacharbrackleft}A{\isadigit{4}}{\isadigit{5}}{\isacharunderscore}P{\isadigit{1}}{\isacharcomma}\ A{\isadigit{4}}{\isadigit{5}}{\isacharunderscore}P{\isadigit{2}}\ {\isasymturnstile}\isactrlsub g\ A{\isadigit{4}}{\isadigit{5}}{\isacharunderscore}C{\isacharbrackright}{\isachardoublequoteclose}\isanewline
\ \ %
\isacommand{unfolding}%
\ A{\isadigit{4}}{\isadigit{5}}{\isacharunderscore}C{\isacharunderscore}def\ A{\isadigit{4}}{\isadigit{5}}{\isacharunderscore}P{\isadigit{1}}{\isacharunderscore}def\ A{\isadigit{4}}{\isadigit{5}}{\isacharunderscore}P{\isadigit{2}}{\isacharunderscore}def\ \isacommand{by}%
\ blast\ %
\isamarkupcmt{proved by tableaus% prover%
}%
%\isanewline
\end{isabellebody}%
%%%%%%%%%%%%%%%%%%%%%%%%%%%%%%%%%%%%%%%%%%%%%

\subsubsection{No Long-term Risk Control (A46)}

\textit{Our social systems and institutions are possibly not capable of controlling risk technologies
on long time scales and of ensuring that they are handled with proper technical care} \parencite{CE}.
Notice that we can make best sense of this objection as (implicitly) presupposing a risk of CE-caused catastrophes (A46-P2).

%%%%%%%%%%%%%%%%%%%%%%%%%%%%%%%%%%%%%%%%%%%%
\begin{isabellebody}\normalsize\isanewline
\isacommand{consts}%
\ RiskControlAbility{\isacharcolon}{\isacharcolon}{\isachardoublequoteopen}w{\isasymRightarrow}bool{\isachardoublequoteclose}\ %
\isamarkupcmt{uses type for (world-lifted) propositions%
}\isanewline
\isacommand{definition}%
\ {\isachardoublequoteopen}A{\isadigit{4}}{\isadigit{6}}{\isacharunderscore}P{\isadigit{1}}\ {\isasymequiv}\ \isactrlbold {\isasymdiamond}\isactrlbold {\isasymnot}RiskControlAbility{\isachardoublequoteclose}\isanewline
\isacommand{definition}%
\ {\isachardoublequoteopen}A{\isadigit{4}}{\isadigit{6}}{\isacharunderscore}P{\isadigit{2}}\ {\isasymequiv}\ \isactrlbold {\isasymnot}RiskControlAbility\ \isactrlbold {\isasymrightarrow}\ \isactrlbold {\isasymdiamond}CEisCatastrophic{\isachardoublequoteclose}\ %
\isamarkupcmt{implicit% premise%
}\isanewline
\isacommand{definition}%
\ {\isachardoublequoteopen}A{\isadigit{4}}{\isadigit{6}}{\isacharunderscore}C\ \ {\isasymequiv}\ \isactrlbold {\isasymdiamond}CEisCatastrophic{\isachardoublequoteclose}%
\isanewline
\isanewline
\isacommand{lemma}%
\ {\isachardoublequoteopen}{\isacharbrackleft}A{\isadigit{4}}{\isadigit{6}}{\isacharunderscore}P{\isadigit{1}}{\isacharcomma}\ A{\isadigit{4}}{\isadigit{6}}{\isacharunderscore}P{\isadigit{2}}\ {\isasymturnstile}\isactrlsub g\ A{\isadigit{4}}{\isadigit{6}}{\isacharunderscore}C{\isacharbrackright}{\isachardoublequoteclose}\ \isacommand{nitpick}%
\ {\isacharbrackleft}satisfy{\isacharbrackright}%
\ %
\isacommand{oops}%
\ %
\isamarkupcmt{model found%: consistent%
}%
\isanewline
\isacommand{lemma}%
\ {\isachardoublequoteopen}{\isacharbrackleft}A{\isadigit{4}}{\isadigit{6}}{\isacharunderscore}P{\isadigit{1}}{\isacharcomma}\ A{\isadigit{4}}{\isadigit{6}}{\isacharunderscore}P{\isadigit{2}}\ {\isasymturnstile}\isactrlsub g\ A{\isadigit{4}}{\isadigit{6}}{\isacharunderscore}C{\isacharbrackright}{\isachardoublequoteclose}\ \isacommand{nitpick}%
\ %
\isacommand{oops}%
\ %
\isamarkupcmt{countermodel found%: not valid%
}%
\isanewline
\end{isabellebody}%
%%%%%%%%%%%%%%%%%%%%%%%%%%%%%%%%%%%%%%%%%%%%%

As before, we can use automated tools to help us find further implicit premises, which may actually correspond to
modifications to the logic of formalization. As shown above by model finder \textit{Nitpick}, the argument A46 is contingent: it is both satisfiable and countersatisfiable. An quick examination of the presented countermodels showed that a stronger base logic, namely modal logic K4, makes A46 \textit{deductive}. This means: we need to postulate the so-called axiom ``4'': $[\vdash {\Box}{\varphi}~{\impl}~{\Box}{\Box}{\varphi}]$ (which can be read intuitively as: ``necessary propositions are so, necessarily''). It is well known that this axiom corresponds to assuming transitivity of the underlying accessibility relation, as we do below. First observe that a (meta-logical) inference can be represented in \textit{Isabelle/HOL} notation as \isakeyword{assumes} \isa{{\isasymphi}\isactrlsub 1 \isakeyword{and} \ \dots\ {\isasymphi}\isactrlsub n \ \isakeyword{shows} {\isasymalpha}}.\footnote{The logic of \textit{Isabelle/HOL} is based upon (higher-order) Gentzen-type natural deduction \parencite{Isabelle}. Our implementation thus handles arguments as (sequent-like) inferences independently from each other. This is different than having the premises for all arguments as axioms in a same theory resp. knowledge-base and drawing conclusions as theorems. In the current approach, our knowledge-base is a set of definitions; two arguments with mutually inconsistent premises will not cause any problems nor trivialize anything. Moreover, conflicting arguments with, \textit{prima facie}, the same premises are also possible; the cause for the conflicting conclusions is to be found in additional (implicit) premises.
}

%%%%%%%%%%%%%%%%%%%%%%%%%%%%%%%%%%%%%%%%%%%%
\begin{isabellebody}\normalsize\isanewline
\isacommand{theorem}%
\ A{\isadigit{4}}{\isadigit{6}}{\isacharunderscore}valid{\isacharcolon}\ \isakeyword{assumes}\ {\isachardoublequoteopen}Transitive\ aRel{\isachardoublequoteclose}
\isanewline~\qquad\qquad\qquad\qquad\quad
\isakeyword{shows}\ {\isachardoublequoteopen}{\isacharbrackleft}A{\isadigit{4}}{\isadigit{6}}{\isacharunderscore}P{\isadigit{1}}{\isacharcomma}\ A{\isadigit{4}}{\isadigit{6}}{\isacharunderscore}P{\isadigit{2}}\ {\isasymturnstile}\isactrlsub g\ A{\isadigit{4}}{\isadigit{6}}{\isacharunderscore}C{\isacharbrackright}{\isachardoublequoteclose}%\isanewline
\isanewline\qquad\qquad
\isacommand{unfolding}%
\ A{\isadigit{4}}{\isadigit{6}}{\isacharunderscore}C{\isacharunderscore}def\ A{\isadigit{4}}{\isadigit{6}}{\isacharunderscore}P{\isadigit{1}}{\isacharunderscore}def\ A{\isadigit{4}}{\isadigit{6}}{\isacharunderscore}P{\isadigit{2}}{\isacharunderscore}def\ \isacommand{using}%
\ assms\ \isacommand{by}%
\ blast
%\isanewline
\end{isabellebody}%
%%%%%%%%%%%%%%%%%%%%%%%%%%%%%%%%%%%%%%%%%%%%%

\subsubsection{CE Interventions are Irreversible (A47)}
As presented by \textcite{CE}, the next two arguments consist of a simple sentence (their conclusion). The present argument states that \textit{CE represents an irreversible intervention}, i.e., that once the first
interventions in world's climate have been set in motion, there is no way to `undo' them. 
We will be working with a predicate logic, and
thus introduce an additional type $e$ for actions (interventions). We also extend our base logic with quantifiers, thus obtaining a quantified modal logic (QML) of formalization (Ex.~\ref{example-sse-QML}).

%%%%%%%%%%%%%%%%%%%%%%%%%%%%%%%%%%%%%%%%%%%%
\begin{isabellebody}\normalsize\isanewline
\isacommand{typedecl}%
\ e\ %
\isamarkupcmt{new type for actions%
}\isanewline
\isacommand{consts}%
\ CEAction{\isacharcolon}{\isacharcolon}{\isachardoublequoteopen}e{\isasymRightarrow}w{\isasymRightarrow}bool{\isachardoublequoteclose}\ %
\isamarkupcmt{uses type for (world-lifted) predicates%
}\isanewline
\isacommand{consts}%
\ Irreversible{\isacharcolon}{\isacharcolon}{\isachardoublequoteopen}e{\isasymRightarrow}w{\isasymRightarrow}bool{\isachardoublequoteclose}
\isanewline
\isacommand{definition}%
\ {\isachardoublequoteopen}A{\isadigit{4}}{\isadigit{7}}{\isacharunderscore}C\ {\isasymequiv}\ \isactrlbold {\isasymforall}I{\isachardot}\ CEAction{\isacharparenleft}I{\isacharparenright}\ \isactrlbold {\isasymrightarrow}\ Irreversible{\isacharparenleft}I{\isacharparenright}{\isachardoublequoteclose}%
\end{isabellebody}%
%%%%%%%%%%%%%%%%%%%%%%%%%%%%%%%%%%%%%%%%%%%%%
%
\subsubsection{No Ability to Retain Options after Irreversible Interventions (A48)}

A48 claims that \textit{irreversible interventions (of any kind) narrow the options of future generations in an unacceptable way}, i.e., it is wrong to carry them out. 

%%%%%%%%%%%%%%%%%%%%%%%%%%%%%%%%%%%%%%%%%%%%
\begin{isabellebody}\normalsize\isanewline
\isacommand{consts}%
\ WrongAction{\isacharcolon}{\isacharcolon}{\isachardoublequoteopen}e{\isasymRightarrow}w{\isasymRightarrow}bool{\isachardoublequoteclose}\ %
\isamarkupcmt{uses type for (world-lifted) predicates%
}\isanewline
\isacommand{definition}%
\ {\isachardoublequoteopen}A{\isadigit{4}}{\isadigit{8}}{\isacharunderscore}C\ {\isasymequiv}\ \isactrlbold {\isasymforall}I{\isachardot}\ Irreversible{\isacharparenleft}I{\isacharparenright}\ \isactrlbold {\isasymrightarrow}\ WrongAction{\isacharparenleft}I{\isacharparenright}{\isachardoublequoteclose}%
\end{isabellebody}%
%%%%%%%%%%%%%%%%%%%%%%%%%%%%%%%%%%%%%%%%%%%%%

\subsubsection{Unpredictable Side-Effects are Wrong (A49)}
\textit{As long as the side-effects of CE technologies cannot be reliably predicted, their deployment is morally wrong.} This argument states two things: (i) CE interventions have unpredictable side-effects, and (ii) CE interventions are wrong. We aim at having (i) entail (ii). An additional (implicit) premise, A49-P2, suggests that interventions with unpredictable side-effects are wrong.

%%%%%%%%%%%%%%%%%%%%%%%%%%%%%%%%%%%%%%%%%%%%
\begin{isabellebody}\normalsize\isanewline
\isacommand{consts}%
\ USideEffects{\isacharcolon}{\isacharcolon}{\isachardoublequoteopen}e{\isasymRightarrow}w{\isasymRightarrow}bool{\isachardoublequoteclose}\ %
\isamarkupcmt{uses type for (world-lifted) predicates%
}\isanewline
\isacommand{definition}%
\ {\isachardoublequoteopen}A{\isadigit{4}}{\isadigit{9}}{\isacharunderscore}P{\isadigit{1}}\ {\isasymequiv}\ \isactrlbold {\isasymforall}I{\isachardot}\ CEAction{\isacharparenleft}I{\isacharparenright}\ \isactrlbold {\isasymrightarrow}\ USideEffects{\isacharparenleft}I{\isacharparenright}{\isachardoublequoteclose}\isanewline
\isacommand{definition}%
\ {\isachardoublequoteopen}A{\isadigit{4}}{\isadigit{9}}{\isacharunderscore}P{\isadigit{2}}\ {\isasymequiv}\ \isactrlbold {\isasymforall}I{\isachardot}\ USideEffects{\isacharparenleft}I{\isacharparenright}\ \isactrlbold {\isasymrightarrow}\ WrongAction{\isacharparenleft}I{\isacharparenright}{\isachardoublequoteclose}\ %
\isamarkupcmt{implicit% premise%
}\isanewline
\isacommand{definition}%
\ {\isachardoublequoteopen}A{\isadigit{4}}{\isadigit{9}}{\isacharunderscore}C\ \ {\isasymequiv}\ \isactrlbold {\isasymforall}I{\isachardot}\ CEAction{\isacharparenleft}I{\isacharparenright}\ \isactrlbold {\isasymrightarrow}\ WrongAction{\isacharparenleft}I{\isacharparenright}{\isachardoublequoteclose}\isanewline
\end{isabellebody}%
%%%%%%%%%%%%%%%%%%%%%%%%%%%%%%%%%%%%%%%%%%%%%
We take the opportunity to illustrate an alternative consequence relation for base logic, namely a \textit{local} consequence relation, exploiting the deduction meta-theorem for modal-like logics (recall the discussion in Def.~\ref{def:SSE}).
%%%%%%%%%%%%%%%%%%%%%%%%%%%%%%%%%%%%%%%%%%%%
\begin{isabellebody}\normalsize\isanewline
\isacommand{abbreviation}%
\ conseq{\isacharunderscore}local{\isadigit{2}}{\isacharcolon}{\isacharcolon}{\isachardoublequoteopen}wo{\isasymRightarrow}wo{\isasymRightarrow}wo{\isasymRightarrow}bool{\isachardoublequoteclose}\ {\isacharparenleft}{\isachardoublequoteopen}{\isacharbrackleft}{\isacharunderscore}{\isacharcomma}{\isacharunderscore}\ {\isasymturnstile}\isactrlsub l\ {\isacharunderscore}{\isacharbrackright}{\isachardoublequoteclose}{\isacharparenright}\ \isakeyword{where}\ \isanewline
\ \ \ \ {\isachardoublequoteopen}{\isacharbrackleft}{\isasymphi}{\isacharcomma}\ {\isasymgamma}\ {\isasymturnstile}\isactrlsub l\ {\isasympsi}{\isacharbrackright}\ {\isasymequiv}\ {\isacharbrackleft}{\isasymturnstile}\ {\isasymphi}\ \isactrlbold {\isasymrightarrow}\ {\isacharparenleft}{\isasymgamma}\ \isactrlbold {\isasymrightarrow}\ {\isasympsi}{\isacharparenright}{\isacharbrackright}{\isachardoublequoteclose}
\isanewline
\isacommand{lemma}%
\ {\isachardoublequoteopen}{\isacharbrackleft}A{\isadigit{4}}{\isadigit{9}}{\isacharunderscore}P{\isadigit{1}}{\isacharcomma}\ A{\isadigit{4}}{\isadigit{9}}{\isacharunderscore}P{\isadigit{2}}\ {\isasymturnstile}\isactrlsub l\ A{\isadigit{4}}{\isadigit{9}}{\isacharunderscore}C{\isacharbrackright}{\isachardoublequoteclose}\ \isacommand{nitpick}%
\ {\isacharbrackleft}satisfy{\isacharbrackright}%
\ %
\isacommand{oops}%
\ %
\isamarkupcmt{model found%: consistent%
}%
\isanewline
\isacommand{theorem}%
\ A{\isadigit{4}}{\isadigit{9}}{\isacharunderscore}valid{\isacharcolon}\ {\isachardoublequoteopen}{\isacharbrackleft}A{\isadigit{4}}{\isadigit{9}}{\isacharunderscore}P{\isadigit{1}}{\isacharcomma}\ A{\isadigit{4}}{\isadigit{9}}{\isacharunderscore}P{\isadigit{2}}\ {\isasymturnstile}\isactrlsub l\ A{\isadigit{4}}{\isadigit{9}}{\isacharunderscore}C{\isacharbrackright}{\isachardoublequoteclose}\isanewline
\ \ %
\isacommand{unfolding}%
\ A{\isadigit{4}}{\isadigit{9}}{\isacharunderscore}C{\isacharunderscore}def\ A{\isadigit{4}}{\isadigit{9}}{\isacharunderscore}P{\isadigit{1}}{\isacharunderscore}def\ A{\isadigit{4}}{\isadigit{9}}{\isacharunderscore}P{\isadigit{2}}{\isacharunderscore}def\ \isacommand{by}%
\ simp\ %
\isamarkupcmt{proved% by simplifier%
}%
\end{isabellebody}%
%%%%%%%%%%%%%%%%%%%%%%%%%%%%%%%%%%%%%%%%%%%%%
%
\subsubsection{Mitigation is also Irreversible (A50)}

Mitigation of climate change (i.e., the `preventive alternative' to CE), too, is, to some
extent, an irreversible intervention with unforeseen side-effects.

%%%%%%%%%%%%%%%%%%%%%%%%%%%%%%%%%%%%%%%%%%%%
\begin{isabellebody}\normalsize\isanewline
\isacommand{consts}%
\ Mitigation{\isacharcolon}{\isacharcolon}e\ %
\isamarkupcmt{uses type for actions/interventions%
}\isanewline
\isacommand{definition}%
\ {\isachardoublequoteopen}A{\isadigit{5}}{\isadigit{0}}{\isacharunderscore}C\ {\isasymequiv}\ Irreversible{\isacharparenleft}Mitigation{\isacharparenright}\ \isactrlbold {\isasymand}\ USideEffects{\isacharparenleft}Mitigation{\isacharparenright}{\isachardoublequoteclose}%
\end{isabellebody}%
%%%%%%%%%%%%%%%%%%%%%%%%%%%%%%%%%%%%%%%%%%%%%
%
\subsubsection{All Interventions have Unpredictable Side-Effects (A51)}

This defense of CE states that we do never completely foresee the consequences of our actions, and thus aims at somehow trivializing the concerns regarding unforeseen side-effects of CE.

%%%%%%%%%%%%%%%%%%%%%%%%%%%%%%%%%%%%%%%%%%%%
\begin{isabellebody}\normalsize\isanewline
\isacommand{definition}%
\ {\isachardoublequoteopen}A{\isadigit{5}}{\isadigit{1}}{\isacharunderscore}C\ {\isasymequiv}\ \isactrlbold {\isasymforall}I{\isachardot}\ USideEffects{\isacharparenleft}I{\isacharparenright}{\isachardoublequoteclose}%
\end{isabellebody}%
%%%%%%%%%%%%%%%%%%%%%%%%%%%%%%%%%%%%%%%%%%%%%
\subsection{Reconstructing the Argument Graph}
\label{sec:case-study:graph}
Recalling our discussion in Section~\ref{sec:arg}, the claim that an argument (or a set of arguments) attacks or supports another argument is, in our approach, conceived as (the claim of) an argument in itself, which also needs to be reconstructed as logically valid (i.e.~as a \textit{deductive} argument, optionally consistent and minimal), either by tweaking the formalization of sentences, or, as we show, by adding additional premises. Those new premises may range from unstated, `implicit' argumentative assumptions, over meaning postulates (of a more or less definitional nature), to axioms and definitions constraining the interpretation of the logical connectives of our underlying (object) logic. A philosophically provocative feature of the \textit{shallow semantical embeddings} (SSE) approach is that it remains agnostic as regards these alternatives, as the only thing it `sees' are expressions in the meta-logic (STT/HOL). This blurres in a sense the distinction between logical and extralogical (resp.,~syncategorematic and categorematic) expressions.

In this spirit, SSEs entitles us to formalize the following question in the meta-logic: ``Does argument(s) A attack (support) argument B?'', even in cases where A and B are formalized using different (object) logics. In case of a negative answer (maybe elicited with the help of automated tools), we have the opportunity to `patch' the corresponding argument (recall Fig.~\ref{figArgAnalysis} and the discussion in Section~\ref{sec:linguafranca}).\footnote{
	This corresponds indeed to an abductive procedure carried out in the meta-logic (STT/HOL). At present, the automation of abductive reasoning for expressive (first- and higher-order) logics is an underdeveloped area. We thus count on leveraging human ingenuity with the help of state-of-the-art automated reasoning technology, in particular proof assistants like \textit{Isabelle/HOL} (cf.~the discussion in \cite[\S 4]{B19}).
}

We will encode support, resp. attack, claims as meta-logical theorems. As usual, they have to become logically entailed from a collection of premises. We thus have e.g., $\varphi_1,\ \dots\ \varphi_n\vdash_{\text{HOL}}\,attacks(A,B)$, where $\varphi_n$ stands for a meta-logical formula (e.g.~a semantical condition on some logical connective) and $\vdash_{\text{HOL}}$ stands for the consequence relation in the higher-order meta-logic.\footnote{Recall that  a (meta-logical) inference is represented in our employed \textit{Isabelle/HOL} notation as \isakeyword{assumes} \isa{{\isasymphi}\isactrlsub 1 \isakeyword{and} \ \dots\ {\isasymphi}\isactrlsub n \ \isakeyword{shows} {\isasymalpha}}.}

\subsubsection{Does A45 support A22?}

When combining several object logics in a same (meta-logical) inference, we employ the corresponding $vld(\cdot)$ predicate (recall Def.~\ref{def:SSE}) as defined for each object logic. For instance, for modal logics the expression $vld(\varphi):=[\vdash\varphi]$ for some term $\varphi$ has been previously defined as global truth, i.e., truth in all worlds, and formalized as $\Pi^w\,\varphi$ (i.e.~$\forall w.\,\varphi~w$), recalling Section~\ref{sec:sse-char}.

In this example we reconstruct a support relation as a deductive inference. Let us recall the corresponding definitions: \textit{A45-C} $\equiv \Diamond\text{CEisCatastrophic}$ (type ${w\typa o}$) and \textit{A22-P1} $\equiv\text{CEisWrong}$ (type $o$).
Observe how the implicit premise becomes formalized in a fragment of the meta-logic (STT/HOL) resulting from the combination of the fragments corresponding to each argument (cf.~Fig.~\ref{figArgAnalysis}).

In this example, as in others, we have utilized three kinds of automated tools integrated
into \emph{Isabelle/HOL}: the model finder \emph{Nitpick} \parencite{Nitpick}, which finds a
counterexample to the claim that A45 supports A22 (without further implicit premises);
the tableau prover \emph{blast},
which can indeed verify that, by adding an implicit
premise (\textit{If CE is possibly catastrophic then its deployment is wrong}), the support relation obtains; and the ``hammer'' tool \emph{Sledgehammer} \parencite{Sledgehammer},
which automagically finds minimal sets of assumptions needed to prove a theorem.

%%%%%%%%%%%%%%%%%%%%%%%%%%%%%%%%%%%%%%%%%%%%
\begin{isabellebody}\normalsize\isanewline	
\isacommand{lemma}%
\ {\isachardoublequoteopen}supports{\isadigit{1}}\ {\isacharbrackleft}{\isasymturnstile}A{\isadigit{4}}{\isadigit{5}}{\isacharunderscore}C{\isacharbrackright}\ A{\isadigit{2}}{\isadigit{2}}{\isacharunderscore}P{\isadigit{1}}{\isachardoublequoteclose}\ \isacommand{nitpick}%
\ %
\isacommand{oops}%
\ %
\isamarkupcmt{countermodel found%
}%
\isanewline
\isacommand{theorem}%
\ \isakeyword{assumes}\ {\isachardoublequoteopen}{\isacharbrackleft}{\isasymturnstile}\ \isactrlbold {\isasymdiamond}CEisCatastrophic\ \isactrlbold {\isasymrightarrow}\ {\isacharparenleft}{\isasymlambda}w{\isachardot}\ CEisWrong{\isacharparenright}{\isacharbrackright}{\isachardoublequoteclose}\ %
\isamarkupcmt{implicit%
}\isanewline
\ \ \ \ \ \ \ \ \isakeyword{shows}\ {\isachardoublequoteopen}supports{\isadigit{1}}\ {\isacharbrackleft}{\isasymturnstile}A{\isadigit{4}}{\isadigit{5}}{\isacharunderscore}C{\isacharbrackright}\ A{\isadigit{2}}{\isadigit{2}}{\isacharunderscore}P{\isadigit{1}}{\isachardoublequoteclose}\isanewline
\ \ %
\isacommand{unfolding}%
\ A{\isadigit{2}}{\isadigit{2}}{\isacharunderscore}P{\isadigit{1}}{\isacharunderscore}def\ A{\isadigit{4}}{\isadigit{5}}{\isacharunderscore}C{\isacharunderscore}def\ \isacommand{using}%
\ assms\ \isacommand{by}%
\ blast\ %
\isamarkupcmt{proved%
}%
\end{isabellebody}%
%%%%%%%%%%%%%%%%%%%%%%%%%%%%%%%%%%%%%%%%%%%%%
%
\subsubsection{Does A46 support A22?}

This reconstruction is similar as the previous, reusing the same implicit premise. Let us recall the definition: \isa{
 {\isachardoublequoteopen}A{\isadigit{4}}{\isadigit{6}}{\isacharunderscore}C\ {\isasymequiv}\ \isactrlbold {\isasymdiamond}CEisCatastrophic{\isachardoublequoteclose
}}.

%%%%%%%%%%%%%%%%%%%%%%%%%%%%%%%%%%%%%%%%%%%%
\begin{isabellebody}\normalsize\isanewline
\isacommand{lemma}%
\ {\isachardoublequoteopen}supports{\isadigit{1}}\ {\isacharbrackleft}{\isasymturnstile}A{\isadigit{4}}{\isadigit{6}}{\isacharunderscore}C{\isacharbrackright}\ A{\isadigit{2}}{\isadigit{2}}{\isacharunderscore}P{\isadigit{1}}{\isachardoublequoteclose}\ \isacommand{nitpick}%
\ %
\isacommand{oops}%
\ %
\isamarkupcmt{countermodel found%
}%
\isanewline
\isacommand{theorem}%
\ \isakeyword{assumes}\ {\isachardoublequoteopen}{\isacharbrackleft}{\isasymturnstile}\ \isactrlbold {\isasymdiamond}CEisCatastrophic\ \isactrlbold {\isasymrightarrow}\ {\isacharparenleft}{\isasymlambda}w{\isachardot}\ CEisWrong{\isacharparenright}{\isacharbrackright}{\isachardoublequoteclose}\ %
\isamarkupcmt{implicit%
}\isanewline
\ \ \ \ \ \ \ \ \isakeyword{shows}\ {\isachardoublequoteopen}supports{\isadigit{1}}\ {\isacharbrackleft}{\isasymturnstile}A{\isadigit{4}}{\isadigit{6}}{\isacharunderscore}C{\isacharbrackright}\ A{\isadigit{2}}{\isadigit{2}}{\isacharunderscore}P{\isadigit{1}}{\isachardoublequoteclose}\ \isanewline
\ \ %
\isacommand{unfolding}%
\ A{\isadigit{2}}{\isadigit{2}}{\isacharunderscore}P{\isadigit{1}}{\isacharunderscore}def\ A{\isadigit{4}}{\isadigit{6}}{\isacharunderscore}C{\isacharunderscore}def\ \isacommand{using}%
\ assms\ \isacommand{by}%
\ simp\ %
\isamarkupcmt{proved% by simplifier%
}%
\end{isabellebody}%
%%%%%%%%%%%%%%%%%%%%%%%%%%%%%%%%%%%%%%%%%%%%%
%
\subsubsection{Do A47 and A48 (together) support A22?}

Here we have diverged from the argument network as introduced in \textcite{CE}, where A48 is rendered as an argument supporting A47. We claim that our reconstruction is more faithful to the given natural language description of the arguments and also better represents their intended dialectical relations. Also notice that an implicit premise is needed to reconstruct this support relation as logically valid, namely that \textit{if every CE action is wrong, then deployment of CE is wrong.}

Let us recall the definitions:
	\isa{
		A{\isadigit{4}}{\isadigit{7}}{\isacharunderscore}C\ {\isasymequiv}\ \isactrlbold {\isasymforall}I{\isachardot}\ CEAction{\isacharparenleft}I{\isacharparenright}\ \isactrlbold {\isasymrightarrow}\ Irreversible{\isacharparenleft}I{\isacharparenright}{\isachardoublequoteclose}} and 
	\isa{
		A{\isadigit{4}}{\isadigit{8}}{\isacharunderscore}C\ {\isasymequiv}\ \isactrlbold {\isasymforall}I{\isachardot}\ Irreversible{\isacharparenleft}I{\isacharparenright}\ \isactrlbold {\isasymrightarrow}\ WrongAction{\isacharparenleft}I{\isacharparenright}{\isachardoublequoteclose}%
}.

%%%%%%%%%%%%%%%%%%%%%%%%%%%%%%%%%%%%%%%%%%%%
\begin{isabellebody}\normalsize\isanewline
\isacommand{lemma}%
\ {\isachardoublequoteopen}supports{\isadigit{2}}\ {\isacharbrackleft}{\isasymturnstile}A{\isadigit{4}}{\isadigit{7}}{\isacharunderscore}C{\isacharbrackright}\ {\isacharbrackleft}{\isasymturnstile}A{\isadigit{4}}{\isadigit{8}}{\isacharunderscore}C{\isacharbrackright}\ A{\isadigit{2}}{\isadigit{2}}{\isacharunderscore}P{\isadigit{1}}{\isachardoublequoteclose}\ \isacommand{nitpick}%
\ %
\isacommand{oops}%
\ %
\isamarkupcmt{countermodel% found%
}%
\isanewline
\isacommand{theorem}%
\ \isakeyword{assumes}\ {\isachardoublequoteopen}{\isacharbrackleft}{\isasymturnstile}\isactrlbold {\isasymforall}I{\isachardot}\ CEAction{\isacharparenleft}I{\isacharparenright}\ \isactrlbold {\isasymrightarrow}\ WrongAction{\isacharparenleft}I{\isacharparenright}{\isacharbrackright}{\isasymlongrightarrow}\ CEisWrong{\isachardoublequoteclose}\ %
\isamarkupcmt{implicit premise%
}\isanewline
\ \ \ \ \ \ \ \ \isakeyword{shows}\ {\isachardoublequoteopen}supports{\isadigit{2}}\ {\isacharbrackleft}{\isasymturnstile}A{\isadigit{4}}{\isadigit{7}}{\isacharunderscore}C{\isacharbrackright}\ {\isacharbrackleft}{\isasymturnstile}A{\isadigit{4}}{\isadigit{8}}{\isacharunderscore}C{\isacharbrackright}\ A{\isadigit{2}}{\isadigit{2}}{\isacharunderscore}P{\isadigit{1}}{\isachardoublequoteclose}\isanewline
\ \ %
\isacommand{unfolding}%
\ A{\isadigit{2}}{\isadigit{2}}{\isacharunderscore}P{\isadigit{1}}{\isacharunderscore}def\ A{\isadigit{4}}{\isadigit{7}}{\isacharunderscore}C{\isacharunderscore}def\ A{\isadigit{4}}{\isadigit{8}}{\isacharunderscore}C{\isacharunderscore}def\ \isacommand{using}%
\ assms\ \isacommand{by}%
\ simp%
\end{isabellebody}%
%%%%%%%%%%%%%%%%%%%%%%%%%%%%%%%%%%%%%%%%%%%%%
%
\subsubsection{Does A49 support A22?}

Observe that in this example we reuse the previous implicit premise.
Let us recall the definition: \isa{
A{\isadigit{4}}{\isadigit{9}}{\isacharunderscore}C\ \ {\isasymequiv}\ \isactrlbold {\isasymforall}I{\isachardot}\ CEAction{\isacharparenleft}I{\isacharparenright}\ \isactrlbold {\isasymrightarrow}\ WrongAction{\isacharparenleft}I{\isacharparenright}{\isachardoublequoteclose
}}.

%%%%%%%%%%%%%%%%%%%%%%%%%%%%%%%%%%%%%%%%%%%%
\begin{isabellebody}\normalsize\isanewline
\isacommand{lemma}%
\ {\isachardoublequoteopen}supports{\isadigit{1}}\ {\isacharbrackleft}{\isasymturnstile}A{\isadigit{4}}{\isadigit{9}}{\isacharunderscore}C{\isacharbrackright}\ A{\isadigit{2}}{\isadigit{2}}{\isacharunderscore}P{\isadigit{1}}{\isachardoublequoteclose}\ \isacommand{nitpick}%
\ %
\isacommand{oops}%
\ %
\isamarkupcmt{countermodel found%
}%
\isanewline
\isacommand{theorem}%
\ \isakeyword{assumes}\ {\isachardoublequoteopen}{\isacharbrackleft}{\isasymturnstile}\isactrlbold {\isasymforall}I{\isachardot}\ CEAction{\isacharparenleft}I{\isacharparenright}\ \isactrlbold {\isasymrightarrow}\ WrongAction{\isacharparenleft}I{\isacharparenright}{\isacharbrackright}{\isasymlongrightarrow}CEisWrong{\isachardoublequoteclose}\ %
\isamarkupcmt{implicit premise%
}\isanewline
\ \ \ \ \ \ \ \ \isakeyword{shows}\ {\isachardoublequoteopen}supports{\isadigit{1}}\ {\isacharbrackleft}{\isasymturnstile}A{\isadigit{4}}{\isadigit{9}}{\isacharunderscore}C{\isacharbrackright}\ A{\isadigit{2}}{\isadigit{2}}{\isacharunderscore}P{\isadigit{1}}{\isachardoublequoteclose}\ \isanewline
\ \ %
\isacommand{unfolding}%
\ A{\isadigit{2}}{\isadigit{2}}{\isacharunderscore}P{\isadigit{1}}{\isacharunderscore}def\ A{\isadigit{4}}{\isadigit{9}}{\isacharunderscore}C{\isacharunderscore}def\ \isacommand{using}%
\ assms\ \isacommand{by}%
\ simp\ %
\isamarkupcmt{proved by simpl.%
}%
\end{isabellebody}%
%%%%%%%%%%%%%%%%%%%%%%%%%%%%%%%%%%%%%%%%%%%%%
%
\subsubsection{Does A50 attack both A48 and A49?}

In this example, the indirect attack towards the main thesis (conclusion of A22) persists, since A47 and A48 jointly support A22 (see above).\footnote{
	Here, again, we diverge from \textcite{CE} original argument network. We argue that, given the natural language description of the arguments, an attack relation between A50 and A48 is better motivated than between A50 and A47 (as originally presented).
}
Also notice that we	employ an additional, implicit premise to reconstruct the attack relation, namely that \textit{mitigation of climate change is not a wrong action}.

Let us recall again the corresponding definitions:

\isa{
	A{\isadigit{5}}{\isadigit{0}}{\isacharunderscore}C\ {\isasymequiv}\ Irreversible{\isacharparenleft}Mitigation{\isacharparenright}\ \isactrlbold {\isasymand}\ USideEffects{\isacharparenleft}Mitigation{\isacharparenright}{\isachardoublequoteclose}},
\isanewline
\isa{
	A{\isadigit{4}}{\isadigit{8}}{\isacharunderscore}C\ {\isasymequiv}\ \isactrlbold {\isasymforall}I{\isachardot}\ Irreversible{\isacharparenleft}I{\isacharparenright}\ \isactrlbold {\isasymrightarrow}\ WrongAction{\isacharparenleft}I{\isacharparenright}{\isachardoublequoteclose}},
\isanewline
\isa{
	A{\isadigit{4}}{\isadigit{9}}{\isacharunderscore}P{\isadigit{2}}\ {\isasymequiv}\ \isactrlbold {\isasymforall}I{\isachardot}\ USideEffects{\isacharparenleft}I{\isacharparenright}\ \isactrlbold {\isasymrightarrow}\ WrongAction{\isacharparenleft}I{\isacharparenright}{\isachardoublequoteclose}}.

%%%%%%%%%%%%%%%%%%%%%%%%%%%%%%%%%%%%%%%%%%%%
\begin{isabellebody}\normalsize\isanewline
\isacommand{lemma}%
\ {\isachardoublequoteopen}attacks{\isadigit{1}}\ {\isacharbrackleft}{\isasymturnstile}A{\isadigit{5}}{\isadigit{0}}{\isacharunderscore}C{\isacharbrackright}\ {\isacharbrackleft}{\isasymturnstile}A{\isadigit{4}}{\isadigit{8}}{\isacharunderscore}C{\isacharbrackright}{\isachardoublequoteclose}\ \isacommand{nitpick}%
\ %
\isacommand{oops}%
\ %
\isamarkupcmt{countermodel found%
}%
\isanewline
\isacommand{lemma}%
\ {\isachardoublequoteopen}attacks{\isadigit{1}}\ {\isacharbrackleft}{\isasymturnstile}A{\isadigit{5}}{\isadigit{0}}{\isacharunderscore}C{\isacharbrackright}\ {\isacharbrackleft}{\isasymturnstile}A{\isadigit{4}}{\isadigit{9}}{\isacharunderscore}P{\isadigit{2}}{\isacharbrackright}{\isachardoublequoteclose}\ \isacommand{nitpick}%
\ %
\isacommand{oops}%
\ %
\isamarkupcmt{countermodel found%
}%
\isanewline
\isanewline
\isacommand{theorem}%
\ \isakeyword{assumes}\ {\isachardoublequoteopen}{\isacharbrackleft}{\isasymturnstile}\ \isactrlbold {\isasymnot}WrongAction{\isacharparenleft}Mitigation{\isacharparenright}{\isacharbrackright}{\isachardoublequoteclose}\ %
\isamarkupcmt{implicit premise%
}\isanewline
\ \ \ \ \ \ \ \ \isakeyword{shows}\ {\isachardoublequoteopen}attacks{\isadigit{1}}\ {\isacharbrackleft}{\isasymturnstile}A{\isadigit{5}}{\isadigit{0}}{\isacharunderscore}C{\isacharbrackright}\ {\isacharbrackleft}{\isasymturnstile}A{\isadigit{4}}{\isadigit{8}}{\isacharunderscore}C{\isacharbrackright}{\isachardoublequoteclose}\isanewline
\ \ %
\isacommand{unfolding}%
\ A{\isadigit{4}}{\isadigit{8}}{\isacharunderscore}C{\isacharunderscore}def\ A{\isadigit{5}}{\isadigit{0}}{\isacharunderscore}C{\isacharunderscore}def\ \isacommand{using}%
\ assms\ \isacommand{by}%
\ blast\ %
\isamarkupcmt{proved% by tableau prover%
}%
\isanewline
\isanewline
\isacommand{theorem}%
\ \isakeyword{assumes}\ {\isachardoublequoteopen}{\isacharbrackleft}{\isasymturnstile}\ \isactrlbold {\isasymnot}WrongAction{\isacharparenleft}Mitigation{\isacharparenright}{\isacharbrackright}{\isachardoublequoteclose}\ %
\isamarkupcmt{implicit premise%
}\isanewline
\ \ \ \ \ \ \ \ \isakeyword{shows}\ {\isachardoublequoteopen}attacks{\isadigit{1}}\ {\isacharbrackleft}{\isasymturnstile}A{\isadigit{5}}{\isadigit{0}}{\isacharunderscore}C{\isacharbrackright}\ {\isacharbrackleft}{\isasymturnstile}A{\isadigit{4}}{\isadigit{9}}{\isacharunderscore}P{\isadigit{2}}{\isacharbrackright}{\isachardoublequoteclose}\ \ \ \ \ \ \ \ \ \ \ \ \ \ \ \ \isanewline
\ \ %
\isacommand{unfolding}%
\ A{\isadigit{4}}{\isadigit{9}}{\isacharunderscore}P{\isadigit{2}}{\isacharunderscore}def\ A{\isadigit{5}}{\isadigit{0}}{\isacharunderscore}C{\isacharunderscore}def\ \isacommand{using}%
\ assms\ \isacommand{by}%
\ blast\ %
\isamarkupcmt{proved% by tableau prover%
}%
\end{isabellebody}%
%%%%%%%%%%%%%%%%%%%%%%%%%%%%%%%%%%%%%%%%%%%%%
%
\subsubsection{Does A51 attack A49?}

Notice that the previous additional premise is again required  to reconstruct the attack relation as deductively valid.
Let us recall the definitions:
\isa{
	A{\isadigit{4}}{\isadigit{9}}{\isacharunderscore}P{\isadigit{2}}\ {\isasymequiv}\ \isactrlbold {\isasymforall}I{\isachardot}\ USideEffects{\isacharparenleft}I{\isacharparenright}\ \isactrlbold {\isasymrightarrow}\ WrongAction{\isacharparenleft}I{\isacharparenright}{\isachardoublequoteclose}} and%
\isa{
	A{\isadigit{5}}{\isadigit{1}}{\isacharunderscore}C\ {\isasymequiv}\ \isactrlbold {\isasymforall}I{\isachardot}USideEffects{\isacharparenleft}I{\isacharparenright}{\isachardoublequoteclose}}.

%%%%%%%%%%%%%%%%%%%%%%%%%%%%%%%%%%%%%%%%%%%%
\begin{isabellebody}\normalsize\isanewline
\isacommand{lemma}%
\ {\isachardoublequoteopen}attacks{\isadigit{1}}\ {\isacharbrackleft}{\isasymturnstile}A{\isadigit{5}}{\isadigit{1}}{\isacharunderscore}C{\isacharbrackright}\ {\isacharbrackleft}{\isasymturnstile}A{\isadigit{4}}{\isadigit{9}}{\isacharunderscore}P{\isadigit{2}}{\isacharbrackright}{\isachardoublequoteclose}\ \isacommand{nitpick}%
\ %
\isacommand{oops}%
\ %
\isamarkupcmt{countermodel found%
}%
\isanewline
\isanewline
\isacommand{theorem}%
\ \isakeyword{assumes}\ {\isachardoublequoteopen}{\isacharbrackleft}{\isasymturnstile}\ \isactrlbold {\isasymnot}WrongAction{\isacharparenleft}Mitigation{\isacharparenright}{\isacharbrackright}{\isachardoublequoteclose}\ %
\isamarkupcmt{implicit premise%
}\isanewline
\ \ \ \ \ \ \ \ \isakeyword{shows}\ {\isachardoublequoteopen}attacks{\isadigit{1}}\ {\isacharbrackleft}{\isasymturnstile}A{\isadigit{5}}{\isadigit{1}}{\isacharunderscore}C{\isacharbrackright}\ {\isacharbrackleft}{\isasymturnstile}A{\isadigit{4}}{\isadigit{9}}{\isacharunderscore}P{\isadigit{2}}{\isacharbrackright}{\isachardoublequoteclose}\isanewline
\ \ %
\isacommand{unfolding}%
\ A{\isadigit{4}}{\isadigit{9}}{\isacharunderscore}P{\isadigit{2}}{\isacharunderscore}def\ A{\isadigit{5}}{\isadigit{1}}{\isacharunderscore}C{\isacharunderscore}def\ \isacommand{using}%
\ assms\ \isacommand{by}%
\ blast\ %
\isamarkupcmt{proved% by tableau prover%
}%
\end{isabellebody}%
%%%%%%%%%%%%%%%%%%%%%%%%%%%%%%%%%%%%%%%%%%%%%
\section{Challenges and Prospects}
\label{sec:conclusion}

The contributions of this article are manifold. Most relevant in our opinion is that we have laid the conceptual foundations for the integration of ideas and techniques from abstract argumentation theory and universal meta-logical reasoning towards the realization of our conception of a computational hermeneutics within a uniform meta-logical environment, thereby creating a momentum of mutual fertilization. Moreover, we have depicted the added value of the integrated system by analyzing and discussing example arguments and argument networks in climate engineering, thereby illustrating how our technology can possibly support the systematic, interactive and partly automated analysis of contemporary discourse in a topical area. 

As a further contribution, this article has introduced a conceptual framework, together with several illustrative examples, enabling an improved characterization and theoretical analysis of the \textit{shallow semantical embeddings} (SSE) technique. We expect that this framework will support, in particular, the development of more elegant faithfulness proofs for SSEs of logics and logic combinations; while these proofs have been very technical and verbose so far, they would now be carried out concisely in few lines only. This is ongoing work which clearly represents an important step towards a systematic method for encoding logics (e.g., modal, deontic, epistemic, paraconsistent) in higher-order logic, a pillar stone of the LogiKEy framework and methodology for designing normative theories in ethical and legal reasoning~\parencite{J48}. Finally, we have contributed a practical demonstration of the maturity of modern theorem proving technology, in particular, the \textit{Isabelle/HOL} proof assistant with its integrated reasoning tools, to support ambitious further work towards a computational hermeneutics.

As regards formal argumentation, some preliminary experiments have shown that the expressivity of higher-order logic indeed allows us to mechanize reasoning with Dung's notions of complete, grounded, preferred and stable semantics in \emph{Isabelle/HOL} and to use its integrated automated tools to carry out computations. This can be very useful for prototyping tasks and as well for reasoning with arguments at the abstract and
structural level in an integrated fashion. Further work is necessary to obtain a satisfactorily usable and scalable implementation.

Concerning the prospects for a fully automated argument reconstruction process, it is worth mentioning that the initial step from natural language to formal representations lies outside our proposed framework. For example, in the presented case study we have `outsourced' the argumentation-mining task to the researchers who carried out the original analysis, while the semantic-parsing task was carried out `manually' by us. However, we are much impressed by recent progress in natural language processing (NLP) for these applications and follow with great interest the latest developments in the argumentation mining community.
Another important challenge concerns the problem of coming up with candidates for additional (implicit) premises that render an inference valid, which is an instance of the old problem of abduction.
The evaluation of candidate formulas is indeed supported by our tool-set, e.g.~model finders can determine (in)consistency automatically, and theorem provers and so-called ``hammers'' help us verify validity using minimal sets of assumptions (also useful to identify `question-begging' ones). The creative part of coming up with (plausible) candidates is, however, still a task for humans in our approach. Abductive reasoning techniques for the kind of expressive logics we work with (e.g. intensional, first- and higher-order) remain, to the best of our knowledge, very limited, so as to support full automation. We could reuse techniques and tools for some less expressive fragments of higher-order logic (in cases where formalized arguments are bound to remain inside those fragments); but in general we strive for the finest granularity level in the semantic analysis, e.g. along the lines of Montague semantics. With all its pros and cons, this is the distinguishing aspect of our approach.

%%% Local Variables: 
%%% TeX-master: "root"
%%% End: 

\begin{acknowledgements}
Acknowledgements will be given in the final version.
\end{acknowledgements}

% Authors must disclose all relationships or interests that 
% could have direct or potential influence or impart bias on 
% the work: 
%
\section*{Conflict of interest}
The authors declare that they have no conflict of interest.

% BibTeX users please use one of
%\bibliographystyle{spbasic}      % basic style, author-year citations
%\bibliographystyle{spmpsci}      % mathematics and physical sciences
%\bibliographystyle{spphys}       % APS-like style for physics
%\bibliography{}   % name your BibTeX data base

\printbibliography

\end{document}